\documentclass{article}

    \usepackage[preprint,nonatbib]{neurips_2022}

\usepackage[utf8]{inputenc} 
\usepackage[T1]{fontenc}    
\usepackage{hyperref}       
\usepackage{url}            
\usepackage{booktabs}       
\usepackage{amsfonts}       
\usepackage{nicefrac}       
\usepackage{microtype}      
\usepackage{xcolor}         

\title{Estimating and Explaining Model Performance When Both Covariates and Labels Shift}

\author{%
  Lingjiao Chen$^1$, Matei Zaharia$^1$, James Zou$^{1,2}$\\
  Department of Computer Science$^1$, Department of Biomedical Data Science$^2$\\
  Stanford University\\
}

\newcommand{\systemnameShiftAtt}{\textsc{SEES}}

\newcommand{\systemnameISS}{\textsc{SJS}}

\newcommand{\Exp}{\mathbb{E}}

\newcommand{\DP}{ \mathbb{P}_s } 
\newcommand{\DQ}{ \mathbb{P}_t } 

\newcommand{\dP}{ {p}_s } 
\newcommand{\dQ}{ {p}_t } 
\newcommand{\dA}{ {p}_a } 

\newcommand{\dPhat}{ \hat{p}_s } 
\newcommand{\dQhat}{ \hat{p}_t }

\newcommand{\nP}{ n_s } 
\newcommand{\nQ}{ n_t } 

\newcommand{\Y}{ {y} } 

\newcommand{\y}{ \bar{{y}} }

\newcommand{\x}{ \bar{\pmb{x}} } 

\newcommand{\X}{ \pmb{x} } 

\newcommand{\z}{ \pmb{z} } 

\newcommand{\dist}{ dd }

\newcommand{\disthat}{ \hat{dd} }

\newcommand{\f}{ \bar{f} }

\newcommand{\WSet}{ \mathcal{W} } 

\newcommand{\w}{ w } 

\newcommand{\what}{ \hat{w} } 

\newcommand{\IndSet}{ I } 

\newcommand{\IndSetJ}{ J }

\newcommand{\sparse}{ m } 

\usepackage{wrapfig}
\usepackage{bbm}
\usepackage{hyperref}       
\usepackage{url}            
\usepackage{booktabs}       
\usepackage{amsfonts}       
\usepackage{microtype}      
\usepackage{microtype}
\usepackage{graphicx}
\usepackage{subfigure}
\usepackage{booktabs} 
\usepackage{amsthm,amsmath,amssymb}
\usepackage{float,url,amsfonts,alltt}
\usepackage{mathtools,rotating}
\usepackage{ifpdf,fancyvrb}
\usepackage{enumitem}

\usepackage{microtype}
\usepackage{graphicx}
\usepackage{booktabs} 
\usepackage{hyperref}

\usepackage{graphicx,xspace,verbatim,comment}
\usepackage{hyperref,array,color,balance,multirow}
\usepackage{balance,float,url,amsfonts,alltt}
\usepackage{mathtools,rotating,amsmath,amssymb}
\usepackage{color,ifpdf,fancyvrb,array}
\usepackage{etoolbox,listings}
\usepackage{bigstrut,morefloats}

\usepackage{pbox}
\usepackage[boxruled,algo2e,linesnumbered]{algorithm2e}
\DeclarePairedDelimiterX{\inp}[2]{\langle}{\rangle}{#1, #2}

\newtheorem{theorem}{Theorem}

\newtheorem{lemma}[theorem]{Lemma}

\newtheorem{definition}{Definition}
\newcommand{\eat}[1]{}

\newcommand{\R}{\mathbb{R}}
\numberwithin{equation}{section}

\usepackage{amsmath}
\usepackage{accents}
\newlength{\dhatheight}

\usepackage{algorithm,algpseudocode}

\usepackage[T1]{fontenc}
\usepackage{upgreek}

\usepackage[greek,english]{babel}

\newcommand{\james}[1]{{\color{blue} {\bf James:} #1}}

\newcommand{\lingjiao}[1]{{\color{orange} {\bf Lingjiao:} #1}}

\begin{document}

\maketitle

\begin{abstract}
Deployed machine learning (ML) models often encounter new user data that differs from their training data. Therefore, estimating how well a given model might perform on the new data is an important step toward reliable ML applications. This is very challenging, however, as the data distribution can change in flexible ways, and we may not have any labels on the new data, which is often the case in monitoring settings. In this paper, we propose a new distribution shift model, Sparse Joint Shift (SJS), which considers the joint shift of both labels and a few features. This unifies and generalizes several existing shift models including label shift and sparse covariate shift, where only marginal feature or label distribution shifts are considered. We describe mathematical conditions under which \systemnameISS{} is identifiable. We further propose \systemnameShiftAtt{}, an algorithmic framework to characterize the distribution shift under \systemnameISS{} and to estimate a model's performance on new data without any labels. We conduct extensive experiments on several real-world datasets with various ML models. 
\eat{\james{Do we have this?}\lingjiao{Unfortunately not yet. I had to remove this.}} Across different datasets and distribution shifts, SEES achieves significant (up to an order of magnitude) shift estimation error improvements over existing approaches.
\end{abstract}

\section{Introduction}\label{Sec:SparseShift:Intro}
Deployed machine learning (ML) models often face new data different from their training data. 
For example, mismatch of deployment-development data in geographical locations~\cite{Wild_Data21}, demographic features~\cite{DemographicEffects2019}, and label balance~\cite{ImbalanceData19} is widely observed and known to affect model performance. 
Thus, estimating and explaining how  a model's  performance changes on the new data is an important step toward reliable ML applications.

Estimating and explaining performance shift is challenging for several reasons, however.
One major challenge is that the data distribution might shift in flexible ways. Another obstacle is that we often do not have labels on the new data, especially in ML monitoring applications. 
Without any assumption on the distribution shift, it's impossible to estimate how well the model would perform on the unlabeled new data.
Previous work often assumes  (i) label shift~\cite{BBSE18}, where feature distributions conditional on the labels are fixed, or (ii) covariate shift~\cite{KLIEP07}, where label distributions conditional on features stay the same.
However, we often do not know whether the real data shift is limited to label or covariate shift, and naively applying estimation methods designed for one shift  may produce inaccurate assessments~\cite{EmpiricalShiftStudy19}.
Moreover, labels and features may shift simultaneously in practice, invalidating these common assumptions.

To tackle the above challenges, here we propose a new distribution shift model, Sparse Joint Shift (\systemnameISS{}), to consider the joint shift of both labels and a few features.
\systemnameISS{} assumes labels and a few features shift, but the remaining features' distribution conditional on the  shifted features and labels is fixed. 
This unifies and  generalizes sparse covariate shift and label shift: both of them are \systemnameISS{}, but some \systemnameISS{} is not label or sparse covariate shift (Figure \ref{fig:SparseShift:IntroOverview}). 
We describe mathematical conditions under which \systemnameISS{} is provably
identifiable: if the non-shifted features are weakly correlated, then the marginal feature distribution uniquely determines the joint distribution under \systemnameISS{}.
This makes it possible to quantify the shift and estimate model performance on new data without any labels.

\begin{figure}[t]
	\centering
	\vspace{-0mm}
	\includegraphics[width=1\linewidth]{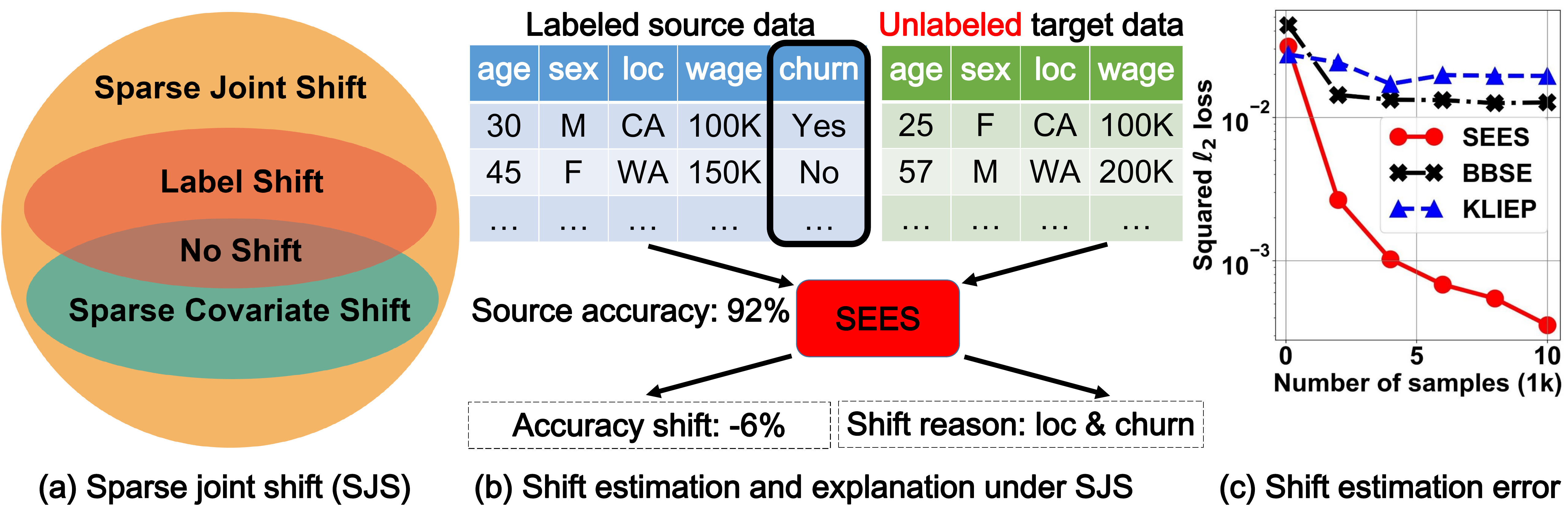}
	\vspace{-0mm}
	\caption{Overview of sparse joint shift (\systemnameISS{}).  (a) Both label shift and sparse covariate shift are \systemnameISS{}, but \systemnameISS{} contains additional shifts as well. 
	(b) illustrates \systemnameShiftAtt{}, a framework for  performance shift estimation and explanation under \systemnameISS{}.   
	Given labeled source and unlabeled target data, \systemnameShiftAtt{} exploits the joint shift modeled by \systemnameISS{} to estimate the model performance change and explain which factors drive the shift. In this example, the goal is to predict \emph{churn}. (c)  \systemnameShiftAtt{} significantly reduces the shift estimation error over existing methods when both labels and covariates shift.}
	\label{fig:SparseShift:IntroOverview}
\end{figure}

\paragraph{When \systemnameISS{} occurs: a motivating example.} Consider Alice, a data scientist who built an ML model for customer churn prediction.
Two years later, churn rate rose in some states but dropped in others, while the distribution of other features given label (churn) and location (state) remained.  
This shift in customer distribution is a natural \systemnameISS{} and challenging to estimate without labels on the new data. 

Furthermore, we propose \systemnameShiftAtt{}, an algorithmic framework for performance \underline{s}hift \underline{e}stimation and \underline{e}xplanation under \underline{S}JS.
\systemnameShiftAtt{} exploits  correlation shifts between features and labels modeled by \systemnameISS{} to improve performance estimation accuracy over existing methods, such as BBSE~\cite{BBSE18} for label shift and KLIEP~\cite{KLIEP07} for covariate shift (Figure \ref{fig:SparseShift:IntroOverview}).
It uses the identified shifted features and labels as a natural explanation of the performance shift.
We present an  extensive empirical evaluation on real world datasets with both simulated \systemnameISS{} and real shifts. 
Our experiments validate the effectiveness of \systemnameShiftAtt{}:  the estimation error of \systemnameShiftAtt{} is often an order of magnitude smaller than that of previous state-of-the-art methods. All together, \textbf{our contributions are}:  
\begin{enumerate}
    \item We formulate and study sparse joint shift (\systemnameISS{}), a new  distribution shift model that considers the joint shift of both labels and features. We  show how it unifies and generalizes existing shift models and when it is identifiable with unlabeled data. 
    \item We propose \systemnameShiftAtt{}, a general framework for performance shift estimation and explanation under \systemnameISS{}.
    We design efficient  substantiations of \systemnameShiftAtt{} for different data types.
    
    \item We provide  comprehensive empirical analysis of our methods. 
    On real world datasets with natural distribution shifts, we found \systemnameShiftAtt{} leads to up to 66\% performance estimation error reduction over standard approaches.

\end{enumerate}

\paragraph{Related Work.} 
\textbf{Label shift:} Label shift has been observed and studied in various domains, such as epidemiologic~\cite{LabelShift_epidemiologic1966}, economics~\cite{LabelShift_econ1977}, and data mining~\cite{forman2008quantifying}.
Recently, there is an increasing interest in  model evaluation and development under under label shift. 
For example, \cite{BBSE18,labelshift_mle_alexandari2020maximum} study how to quantify label shift and evaluate an ML model's performance under label shift accordingly. 
\cite{LabelShift_active2021} gives an algorithm for active learning under label shift. 
Online adaptation to label shift 
\cite{LabelShift_Online2021} is also shown to be feasible. Label shift can be viewed as a subset of our proposed shift model \systemnameISS{}. 
Although simpler and easier to estimate, label shift does not capture the joint shift of labels and features.
Thus, methods developed for label shift may not work well under \systemnameISS{}.
 
\textbf{Covariate shift:}
Covariate shift~\cite{covariate_firstpaper_shimodaira2000improving,covaraite_distcompare_yamada2011relative,covariateshift_unify_moreno2012unifying} is perhaps the most widely adopted assumption in data distribution shift. Since studied in the seminal work~\cite{covariate_firstpaper_shimodaira2000improving}, various methods have been proposed to estimate covariate shift, including  KLIEP~\cite{covariate_sugiyama2008direct}, KMM~\cite{covariate_gretton2009}, and IWCV~\cite{covariate_sugiyama2007}. 
Adaptation to covaraite shift has been found useful in many applications, such as 
spam filtering~\cite{covariate_application_spam_bickel2006dirichlet}, emotion recognition~\cite{covariate_emotion_jirayucharoensak2014eeg} and human activity detection~ \cite{covaraite_applications_humanactivity_hachiya2012importance}.
More recently, covariate shift adaptation is jointly optimized with  model robustness~\cite{covariate_robustness_schneider2020}, fairness~\cite{covariate_fairness_rezaei2021}, and conformal prediction~\cite{covariate_conformal_tibshirani2019}.


\textbf{Unsupervised model performance evaluation}:
Model performance evaluation without labels has received relatively limited attention. Domain-specific models' performance can be estimated via certain statistics, such as confidence score~\cite{modeleval_confidence_guillory2021predicting}, rotation prediction~\cite{ModelEval_rotation_deng2021does} and feature statistics of the datasets sampled from a meta-dataset~\cite{ModelEval_featurestats_deng2021labels} for image recognition.
General model evaluation often relies on different assumptions and accessibility~\cite{ModelEval_chen2021detecting, chen2021did,ModelEval_mandoline_chen2021, ModelEval_MIT_chuang2020estimating,ModelEval_donmez2010unsupervised,guillory2021predicting,jiang2019fantastic,ModelEval_Lazy_welinder2013}. 
For example, \cite{ModelEval_mandoline_chen2021} assumes covariate shift and requires users to provide an approximation (slice) of the shifted features, while \cite{ModelEval_chen2021detecting} needs white-box access to the ML models to train an ensemble as a reference.
\cite{ModelEval_donmez2010unsupervised} assumes the label distributions are known, while \cite{ModelEval_MIT_chuang2020estimating} needs a feature independence structure given the labels.
When a small number of labels can be obtained, 
\cite{ModelEval_Lazy_welinder2013} proposes an active model evaluation approach. 
To our knowledge, this is the first paper that explicitly models the joint shift of both labels and a few features with provable identification guarantees. 
Moreover, we do not require access to side information such as model design or metadata.

\section{Preliminaries  and Problem Statement}\label{Sec:SparseShift:Preli}
We start by giving the preliminaries and   the problem of estimating and explaining  performance shift.


\paragraph{Prediction tasks and ML models.} In this paper we consider the standard classification task: given a $d$-dimensional feature vector $\pmb{x} \in \mathcal{X}\subseteq \R^{d}$ from the feature space $\mathcal{X}$, the goal is to predict its associated label $y \in \mathcal{Y}$ in the label space $\mathcal{Y}$. 
Let $f(\cdot): \R^{d} \mapsto \mathcal{Y}$ denote an ML model designed for such a task. 
For simplicity, we assume that $\mathcal{Y}=\{1,2,\cdots, L\}$.
Given the model's prediction $f(\pmb x)$ and its true label $y$, its performance is quantified by some loss function $\ell(\cdot,\cdot)$.
A popular choice is the standard 0-1 loss: $\ell(a,b) = \mathbbm{1}_{a=b}$, which we focus on, but other losses are also applicable. 

\paragraph{Joint distribution  shift.} 
The training and inference data for ML often come from two different distributions, referred to as source domain and target domain.
Here, we consider the general case when the joint distribution vary across the source and target domains, and  call this \textit{joint distribution shift}. Formally, let $\DP, \DQ:\mathcal{X}\times \mathcal{Y} \mapsto [0,1]$ denote the source and target domains, and  $\dP,\dQ$ be their probability density (or mass) functions.
Then joint distribution shift means $\dP(\X,y)\not=\dQ(\X,y)$.

\paragraph{Problem statement.} Suppose we are given a labeled dataset $D_{s} \triangleq \{(\X^{s,i}, y^{s,i})\}_{i=1}^{\nP}$ from the source distribution $\DP$,  an unlabeled dataset $D_{t} \triangleq \{(\X^{t,i})\}_{i=1}^{\nQ}$ from the target distribution ${\DQ}$, and an ML model $f(\cdot)$ predicting the associated label given any feature vector $\X$.
Our goal is to estimate how much performance changes from the source domain to the target domain.
More formally, we aim at estimating the performance shift $\Delta \triangleq \Exp_{(\X,y)\sim \DQ}[\ell(f(\X),y)] - \Exp_{(\X,y)\sim \DP}[\ell(f(\X),y)]$.
This is challenging as we do not observe labels on the target domain.
In many applications, attributing the performance shift to certain features is also desired. Thus, we are also interested in identifying a set of features to explain the performance shift. 
\section{\systemnameISS{}: A Tractable Unification of Label Shift and Sparse Covariate Shift}\label{Sec:SparseShift:SparseShiftTheory}

At a first glance, estimating the performance shift under joint distribution shift without observing labels from target domain seems hopeless: if the marginal feature distributions are identical for both domains, then observing the features alone should give 0 as the estimated performance shift. However, the label distribution given any feature on the target domain is arbitrary, and thus the estimated shift can be arbitrarily bad. 
In other words, joint distribution shift is not identifiable with no target labels.


To mitigate nonidentifiability, it's necessary to make additional assumptions. 
The most popular assumptions in literature are label shift~\cite{BBSE18} and covariate shift~\cite{covariate_firstpaper_shimodaira2000improving}.
Label shift assumes that only label distribution may change, but the feature distribution given a label remains, i.e.,  $\dP(\X|y)=\dQ(\X|y)$.
On the other hand, covariate shift assumes that feature distribution can shift, but the label distribution given the features is fixed, i.e.,    $\dP(y|\X)=\dQ(y|\X)$. 
However, those assumptions disallow simultaneous changes of both features and labels, which often happen in real-world data~\cite{Wild_Data21,recht2019imagenet,taori2020measuring}. 
To enable joint feature and label estimation which is tractable, we introduce a subclass of joint distribution shift, \textit{Sparse Joint Shift} (\systemnameISS{}), as follows.

\begin{definition}[Sparse Joint Shift (\systemnameISS)]
Suppose for an integer $\sparse\leq d$ and an index set 
$\mathcal{I}\subset [d]$ with size at most $\sparse$ (i.e., $|\IndSet|\leq \sparse$), 
$\dP(\X_{\IndSet^c}|\X_\IndSet, y)=\dQ(\X_{\IndSet^c}|\X_\IndSet, y)$.
Then we say the source and target pair $(\dP,\dQ)$ is under $\sparse$-Sparse Joint Shift, or $\sparse$-\systemnameISS{}. 
Here, $\IndSet^c\triangleq [d]-\IndSet$. We call  $\IndSet$ the shift index set.
\end{definition}

Roughly speaking, \systemnameISS{} allows both labels and a few features to shift, but assumes the remaining features' conditional distribution to stay the same.
Section \ref{Sec:SparseShift:Intro} gives one example when \systemnameISS{} occurs, and more examples and discussions can be found in the appendix. 
Next, we will study when this assumption allows tractable performance shift estimation. 
All proofs are left to the appendix.

\subsection{When is sparse joint shift identifiable?}
Recall that additional assumptions are needed because the general joint distribution shift is not identifiable.
However, when $\sparse=d$, $\sparse$-\systemnameISS{} simply becomes general joint distribution shift. 
Thus, it is worthy understanding when $\sparse$-\systemnameISS{} resolves the identifiability issue.
To do so, let us first formally introduce the notation of identifiability.

\begin{definition}[Identifiable]
Suppose the source-target tuple $(\dP, \dQ)$ is under $\sparse$-\systemnameISS{}. $(\dP, \dQ)$ is  identifiable if and only if for any alternative distribution  $p_a(\X,\Y)$, if $p_a(\X) = \dQ(\X)$ and $\exists \mathcal{J}\subset [d], |\mathcal{J}|\leq \sparse$, such that $ p_a(\X_{\mathcal{J}^c} |\X_\mathcal{J},y) =\dP(\X_{\mathcal{J}^c} |\X_\mathcal{J},y)$, then $p_a(\X,y)=\dQ(\X,y)$. 
\end{definition}
The identifiability can be easily interpreted in words: If a joint feature and label distribution matches the target feature distribution and satisfies the $\sparse$-\systemnameISS{} requirement together with the source distribution, it has to be the target distribution.
The following statement shows when $(\dP,\dQ)$ is identifiable. 

\begin{theorem}\label{thm:SparseShift:identifiable}
Suppose $(\dP,\dQ)$ is under $\sparse$-\systemnameISS{}.
Assume for any set $\mathcal{J} \subset [d], |\mathcal{J}|\leq \sparse$ and any fixed $\x \in \mathcal{X}$, the probability density (or mass) functions $\{\dP(\X_{\IndSetJ^c\cap\IndSet^c},\X_{\mathcal{J}\cup\IndSet}=\x_{\mathcal{J}\cup \IndSet},y=i)\}_{i=1}^{L}$ are linearly independent.
Then $(\dP, \dQ)$ is  identifiable.

\end{theorem}

This statement sheds light on why uniquely identifying the target distribution without target label is feasible  under sparse joint shift.
Roughly speaking, $\sparse$-\systemnameISS{} requires that given the shifted features  and labels, the remaining features' distribution remains the same on both domains. 
If those remaining features are different enough (linear independence), they can uniquely determine the distribution of the shifted features and labels.
We stress that the linear independence is necessary: if it does not hold, then for any $\sparse$, we can always find some source-target pair $(\dP,\dQ)$ which is not identifiable. 
Linear independence implicitly requires sparsity: if $\sparse>d/2$, then $\IndSetJ^c \cap \IndSet^c$ can be empty and the linear independence does not hold.  
In other words, the sparsity is necessary for the shift to be identifiable.


\subsection{How does \systemnameISS{} relate to label shift and covariate  shift?}
A natural question is how does SJS relates to standard label shift and covariate shift. 
To answer this, let us first introduce label and sparse covariate shift formally. 
\begin{definition}
The source and target  $(\dP,\dQ)$ is under \textit{Label Shift} iff $\dP(\X|\Y)=\dQ(\X|\Y)$, and under $\sparse$-\textit{Sparse Covariate Shift} iff $\dP(\X_{I^c},\Y|\X_I)=\dQ(\X_{I^c},\Y|\X_I)$ for some index set $I$ with size $\sparse<d$. 
\end{definition}

Now we are ready to answer the above question.
\begin{theorem}\label{thm:SparseShift:connection}
If $(\dP,\dQ)$ is under label shift, then it is also under  $0$-\systemnameISS{}.
If $(\dP,\dQ)$ is under $\sparse$-sparse covariate shift, then it is also under $\sparse$-\systemnameISS{}.
In addition, there exists $(\dP,\dQ)$ under $\sparse$-\systemnameISS{} such that it is under neither label shift or covariate shift.
\end{theorem}

There are several takeaways. 
First, label shift implies \systemnameISS{} without additional requirements.
In fact, as certain distribution pairs are under \systemnameISS{} but not label shift, \systemnameISS{} is strictly more general than label shift. 
Second, \systemnameISS{} also includes sparse covariate shift. 
When $\sparse=d$, \systemnameISS{} completely unifies both label shift and covariate shift, though it is not identifiable. Identifiable \systemnameISS{}, on the other hand, unifies label shift and sparse covariate shift.
Finally, \systemnameISS{} also allows shifts not covered by label shift and covaraite shift: the correlation between label and (a set of) features can be shifted.  

\section{Shift Estimation and Explanation under Sparse Joint Shift}\label{Sec:SparseShift:theory}

\begin{figure*}[t]
	\centering
	\vspace{-0mm}
	\includegraphics[width=0.99\linewidth]{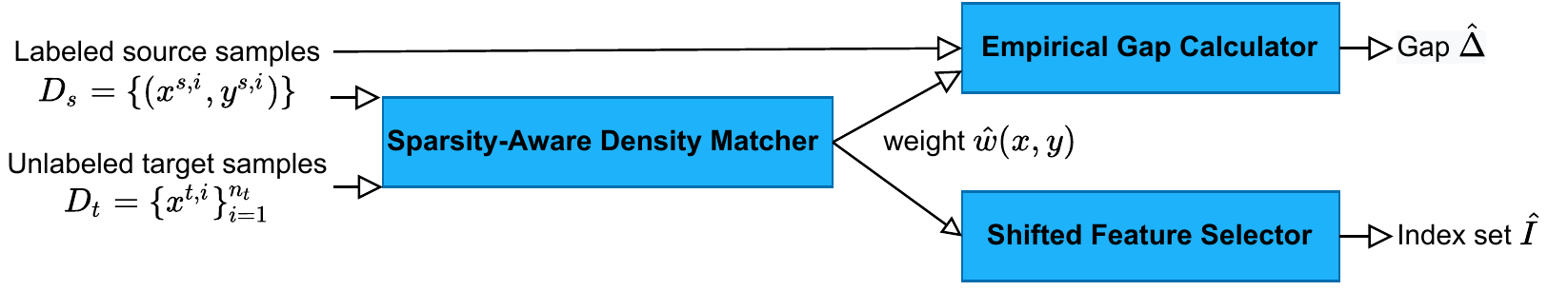}
	\vspace{-0mm}
	\caption{How \systemnameShiftAtt{} works. Given labeled source and unlabeled target data, \systemnameShiftAtt{}  uses a sparsity-aware density matcher to learn a weight function ${\what(\x,y)}$. 
	Next, an empirical gap calculator computes the performance gap $\hat{\Delta}$ by  weighing the source samples with the learned ${\what(\x,y)}$. 
	The shifted feature selector extracts the  features $\hat{I}$ on which the weight function depends heavily.  	
	}
	\label{fig:SparseShift:MainFlow}
\end{figure*}
Now we present \systemnameShiftAtt{}, an algorithmic framework to estimate and explain the performance shift $\Delta$ when the source and target domain is under $\sparse$-\systemnameISS{}. As shown in Figure \ref{fig:SparseShift:MainFlow}, it consists of three components: a sparsity-aware density matcher, an empirical gap calculator, and a shifted feature selector.
Given the labeled source samples and unlabeled target samples, we first adopt the sparsity-aware density matcher to obtain an estimated ratio of the target and source  density functions, denoted by $\hat{w}(\X,y)$.
Next, the empirical gap calculator is responsible to estimate the performance shift $\Delta$ via appropriately reweighting source samples with ratio $\what(\X,y)$.  
Finally, the shifted feature selector picks a set of features as the explanation for the shift.
We explain each component as follows.
\subsection{Sparsity-aware density matcher}
A key component of \systemnameShiftAtt{} is the sparsity-aware density matcher.
Here we want to find some weight function $\what(\X,y)$ to induce an estimated target distribution $\dQhat(\X,y) \triangleq \what(\X,y) \cdot \dP(\X,y)$.
Our goals include (i) a small distance between the estimated target distribution and the true target distribution, (ii)  $\sparse$-\systemnameISS{} between the source and the estimated target distributions, and (iii) flexible parameterization of the weight function.
To achieve those goals, we propose the following  optimization framework
\begin{equation}\label{Problem:SparseShift:OptimizationFramework}
\begin{split}
\min_{\w(\X,y)\in \WSet} & \textit{ } D(\dQ(\X),\dQhat(\X))\\
\mbox{s.t.} & \textit{ } \dQhat(\X) = \sum_{y=1}^{L}\w(\X,y) \cdot \dP(\X,y),\mbox{ and } 
 \w(\X,y) \mbox{ depends on at most $\sparse$ features of $\X$.} \\
\end{split}
\end{equation}
Here, $D(\cdot,\cdot)$ is some distance metric that measures the difference between two density functions. 
We minimize the distance between the induced feature density $\dQhat(\X)$ and the target feature density $\dQ(\X)$. The minimization is not over joint label and feature  distributions since target labels are not available.
The induced feature density function can be easily derived from source density function and the weight function, encoded in the first constraint. 
$\sparse$-\systemnameISS{} is enforced by the second constraint: $\sparse$-\systemnameISS{} means given $\sparse$ features and labels, the distributions of remaining features are fixed across source and the induced domain, which holds if and only if their density ratio $\w(\X,y)$ only depends on those $\sparse$ features.  
$\WSet$ represents the set of all feasible weight functions.
Different parameterization can be easily realized by adopting different $\WSet$. 
Assume access to density functions $\dP(\X,y)$ and $\dQ(\X)$, and a weight function set $\WSet$ containing the true weight $\w^*(\X,y)\triangleq \frac{\dQ(\X,y)}{\dP(\X,y)}$.
One can easily show the above optimization returns the true weight function $\w^*(\X,y)$ for identifiable $\sparse$-\systemnameISS{}.

One benefit of the above framework is the flexibility of concrete instantiations. 
Different design choices, including the distance metric $D(\cdot,\cdot)$ and the weight parameterization space $\WSet$, can fit different feature types, incorporate domain knowledge, and tradeoff different sample and computational complexity. 
We give two instantiations of the above optimization:
\systemnameShiftAtt{}-c for continuous features, and \systemnameShiftAtt{}-d for discrete features.

\paragraph{\systemnameShiftAtt{}-c: SEES for continuous features.}
For continuous features, we  use KL-divergence $D_{KL}(\cdot,\cdot)$ as the distance metric, and initialize the parameterization space $\WSet$ by linear combinations of $K$ fixed basis functions, $\phi_1(\X,y),\phi_2(\X,y),\cdots, \phi_K(\X,y)$.
That is to say, $\WSet=\{\w(\X,y)|\w(\X,y)=\sum_{k=1}^{K} a_{k,y} \phi_k(\X,y), a_{k,y}\geq 0, \Exp_{(\X,y) \sim  \DP}\left[\sum_{k=1}^{K}a_{k,y} \phi_k(\X,y)\right]=1, \alpha_{k,y}\geq 0  \}$. 
The last two constraints ensure $\w(\X,y)\cdot \dP(\X,y)$ is a valid probability density. 
Those basis functions encode users' prior knowledge about the shift.
A simple choice, for example,  is linear functions (when $\X_k\geq 0$): setting $K=d$ and $\phi_k(\X,y)=\X_k$.
To model the dependence on different features, let $e_i$ denote all indexes $k$ such that  $\phi_k(\cdot)$ depends on feature $\X_i$, and introduce a vector $\pmb \beta\in\R^{d}$ such that $\pmb \beta_i\triangleq\sqrt{\sum_{k\in e_i}^{} \sum_{y=1}^{L} a_{k,y}^2}$.
The feature dependence requirement in Problem \ref{Problem:SparseShift:OptimizationFramework} is equivalent to sparsity constraint $\|\pmb \beta\|_0\leq s$. 
We can relax the $0$-norm by $1$-norm and obtain one instantiation  as

\begin{equation*}
\begin{split}
    \max_{a_{1,1},a_{1,2},\cdots,a_{K,L}} & \Exp_{\X \sim \DQ}\left[\log \sum_{y=1}^{L} \dP(y|\X) \sum_{k=1}^{K} a_{k,y} \phi_k(\X,y)\right] + \eta  \sum_{i=1}^{d}\sqrt{\sum_{k\in e_i}^{} \sum_{y=1}^{L} a_{k,y}^2} \\
    \textit{s.t. } &
    \Exp_{(\X,y) \sim  \DP}\left[\sum_{k=1}^{K}a_{k,y} \phi_k(\X,y)\right]=1, \alpha_{k,y}\geq 0 
\end{split}
\end{equation*}
where $\eta>0$ controls the trade-off between sparsity and the KL distance. One benefit of this instantiation is computational efficiency: the constraint is linear in the optimization variables, and the objective is convex.
Thus, the problem is convex and can be  efficiently solved.
The label distribution given feature $\dP(y|\X)$ is unknown but can be approximated by the ML model $f(\cdot)$ trained on the source domain.
Given finite samples, the expectations can be replaced by their empirical estimation.

\paragraph{\systemnameShiftAtt{}-d: \systemnameShiftAtt{} for discrete features.}
Another interesting instantiation exists for discrete features.
With no prior  knowledge, we parameterize $\WSet$ to include all possible $\sparse$-\systemnameISS{}: specifically, $\WSet$ contains all  tuple $(J,\w_J(\X_J,y))$,  where index set $J\subset [d], |J|=\sparse$ represents the shifted features, and weight function $\w_J(\X_J,y)$ only depends on $\X_J$ and $y$.

Features only take finite values, so we can view the density (mass) functions as vectors with finite dimensions. 
Thus, we adopt the squared $\ell_2$ distance, i.e., $D(\pmb z,\pmb z')=\sum_{i=1}^{|\pmb z|}(\pmb z_i-\pmb z_i')^2$ to measure distance.
However, naively measuring $\ell_2$ distance between $\dQ(\X)$ and $\dQhat(\X)$ leads to a computational complexity exponential in $d$.
Instead, we measure the distance on a set of marginal densities: given an index set $J$, for every index set with size $2s$ that contains $J$, denoted by $\kappa$, we measure the squared $\ell_2$ distance between $\dQ(\X_\kappa,f(\X))$ and $\dQhat(\X_\kappa,f(\X))$, and then aggregate over $\kappa$. 
This design leads to the following instantiation of Problem \ref{Problem:SparseShift:OptimizationFramework}

\begin{equation}\label{Prob:SparseShift:discretefull}
    \begin{split}
    \min_{J,\w_J(\X,y)}& \sum_{\kappa:J\subseteq\kappa,|\kappa|=2 \sparse  }^{}\sum_{\bar{f}=1  }^{L}\sum_{
    \x_\kappa \in \mathcal{X}_\kappa    }^{}\|\dQ(\x_\kappa,\bar{f}) - \sum_{\bar{y}=1}^{L} \w_J(\x_J,\bar{y}) \cdot \dP( \x_\kappa, \bar{f}, \bar{y}) \|_2^2, 
    \textit{ } s.t.  |J|=\sparse \\     
    \end{split}
\end{equation}
where $\dQ(\x_\kappa,\bar{f})$ and $\dP(\x_\kappa,\bar{f},\bar{y})$ are short for $\dQ(\X_\kappa={\x}_\kappa,f(\X)=\bar{f})$ and $\dP(\X_\kappa=\x_\kappa,f(\X)=\bar{f}, y=\bar{y})$, respectively.

Compared to the naive approach, the above formulation is much more computationally efficient: the number of parameters in the above objective is only polynomial in the feature dimension $d$.
For fixed $J$, the problem is simply a linear regression over the weight $\w_J(\X_J,y)$ and thus can be efficiently solved.
In practice, one can estimate $\dQ(\x_\kappa,\bar{f})$ and $\dP(\x_\kappa,\bar{f},\bar{y})$ via  labeled source and unlabeled target samples, and then solve the empirical version of the above problem.
Compared to using KL-divergence,
 solving the empirical version produces the correct shifted index set and a weight function close to the true weight $\w^*(\X,y)$ (under mild conditions). This is formally stated as follows.
\begin{theorem}\label{Thm:SparseShift:AlgConvergence}
Consider when all features are discrete, i.e., for each $i$, $\X_i\in\{1,2,\cdots, v\}$. Suppose (i) the source and target are under exact $\sparse$-\systemnameISS{}, (ii) for any set $\mathcal{J} \subset [d], |\mathcal{J}|\leq \sparse$ and any ${\x}\in \mathcal{X}$, the   marginal probability density (or mass) functions $\{\dP(f(\X),\X_{\mathcal{J}\cup\IndSet}=\x_{\mathcal{J}\cup \IndSet},y=i)\}_{i=1}^{d}$ are linearly independent, and (iii) $w(\X,y)$ is bounded by a constant $M$. Then there exists some constant $c$ (independent of $d, \nP$ and $\nQ$), such that if $ \sqrt{\frac{1}{2\nP}} +
 LM \sqrt{\frac{1}{2\nQ}} < c/\sqrt{2 \sparse \log d + \sparse \log {v} +2 \log L + \log 1/\delta}  $,  then with probability $1-\delta$, (i)  the index set $\hat{J}$ learned by Problem \ref{Prob:SparseShift:discretefull} matches the true shift index set $I$, and (ii) the produced weights $\w_{\hat{J}}(\X_{\hat{J}},y)$ satisfies  $\left|\w_{\hat{J}}(\X_{\hat{J}},y) -\w^*(\X,y)\right| \leq O\left(\sqrt{2 \sparse \log d + \sparse \log {v} +2 \log L + \log 1/\delta} \left( \sqrt{\frac{1}{2\nP}} +
 LM \sqrt{\frac{1}{2\nQ}} \right)\right)$.
\end{theorem}
Roughly speaking, this statement ensures that, 
when source and target sample sizes are large enough, with high probability, the true shift index set can be identified with finite samples, and the error rate of the learned weight function is approximately the inverse of sample sizes' square root. 

\paragraph{Comparisons of \systemnameShiftAtt{}-c and \systemnameShiftAtt{}{}-d.} \systemnameShiftAtt{}-d enjoys mathematical guarantees, but \systemnameShiftAtt{}-c  can be computationally more efficient. 
In practice, we can  discretize continuous features to use \systemnameShiftAtt{}-d. 

\subsection{Empirical gap calculator and shifted feature selector}
Now we explain how the other two components of \systemnameShiftAtt{} work. 
The empirical gap calculator computes the performance shift $\hat{\Delta}$ via three steps. First, it estimates the source performance by $\frac{1}{\nP}\sum_{i=1}^{\nP} \ell(\X^{s,i},y^{s,i})$.
Next, it estimates the performance on the induced target distribution.
Note that the performance on the induced target domain is $ \int \dQhat(\X,y) \ell(\X,y)  d\X dy = \int \hat{\w}(\X,y)\dP(\X,y) \ell(\X,y)  d\X dy =\Exp_{(\X,y)\sim \DP} \left[ \hat{\w}(\X,y) \ell(\X,y)\right]$.
Thus, we use the weighted loss on the source samples $\frac{1}{\nP}\sum_{i=1}^{\nP} \hat{w}(\X^{s,i},y^{s,i}) \ell(\X^{s,i},y^{s,i})$ as the estimation.
Finally, their  difference, i.e.,  $\hat{\Delta} = \frac{1}{\nP}\sum_{i=1}^{\nP} (\hat{w}(\X^{s,i},y^{s,i})-1) \ell(\X^{s,i},y^{s,i})$ is returned as the estimated performance shift.


The shifted feature selector picks a set of features as the shift explanation. 
For discrete data, the weight function $\what(\X,y)$ learned by the density matcher's instantiation is 
parameterised  as a shifted index $\hat{J}$ and the corresponding weight $\what_{\hat{J}}(\X_{\hat{J}},y)$.    
Thus, a natural choice is to return $\hat{I}=\hat{J}$ as the explanation. 
For continuous data, the weight function is  $\what(\X,y) =\sum_{k=1}^{K} \hat{a}_{k,y} \phi_k(\X,y)$, where $\hat{a}_{k,y}$ is learned by the corresponding instantiation.  
Recall that $e_i$ denotes all basis functions that depends on feature $i$. 
Then  $\hat{\pmb \beta}_i\triangleq\sqrt{\sum_{k\in e_i}^{} \sum_{y=1}^{L} \hat{a}_{k,y}^2}$ can be viewed as the total contribution of feature $i$. Thus, a simple choice is to pick features with the  $\sparse$ largest contributions. Formally, we use $\hat{I} = \{i| \hat{\pmb \beta}_i > \hat{\pmb \beta}_{(d-\sparse)}\}$, where $\hat{\pmb \beta}_{(d-\sparse)}$  is the $d-\sparse$ smallest value in $\hat{\pmb \beta}$.
\section{Experiments}\label{Sec:SparseShift:Experiment}
In this section, we study the performance of \systemnameShiftAtt{} on several real world datasets with synthetic and natural distribution shifts.
Our goal is four-fold:  (i) understand when and how \systemnameShiftAtt{} estimates the performance shift, (ii) evaluate the trade-offs between the estimation performance reached by \systemnameShiftAtt{} and the required dataset sizes, (iii) explore the effects of shift sparsity on \systemnameShiftAtt{}'  performance, and (iv) validate the effectiveness of  \systemnameShiftAtt{}  on datasets with real world distribution shifts.

\paragraph{Datasets, ML models and baselines.} 
Six datasets are used for evaluation purposes.
we first simulate various \systemnameISS{} on BANKCHURN~\cite{dataset_bankchurn}, COVID-19~\cite{dataset_covid19}, and CREDIT~\cite{dataset_credit_yeh2009comparisons} to systematically understand the performance of \systemnameShiftAtt{}.
Next, we apply \systemnameShiftAtt{} on EMPLOY, INCOME, and INSURANCE~\cite{dataset_census_ding2021retiring} with real world distribution shifts and perform an in-depth analysis.
We use a gradient boosting tree model as the ML model, and results for more models can be found in the Appendix. For comparison, we adopt two state-of-the-art methods for comparison: BBSE~\cite{BBSE18} for label shift and KLIEP~\cite{KLIEP07} for covariate shift. More details on the experiments can be found in the Appendix.

\paragraph{Case study.} We start with a case study on the dataset COVID-19.  The task is to predict whether a person tests positive for COVID-19. We simulate a joint shift of the feature aged and label. 
Specifically, both source and target data contain 5000 young and aged individuals.
The positive rate is 40\% for both young and aged group from the source. 
In the target data, we raise the positive rate to 80\% for aged group and 50\% for young group.
This simulates a shift due to a COVID variant more harmful to the elder than its ancestor. 
We adopt \systemnameShiftAtt{}-d as all features are categorical.

\begin{figure}[t]
    \centering
    \includegraphics[width=1.0\linewidth]{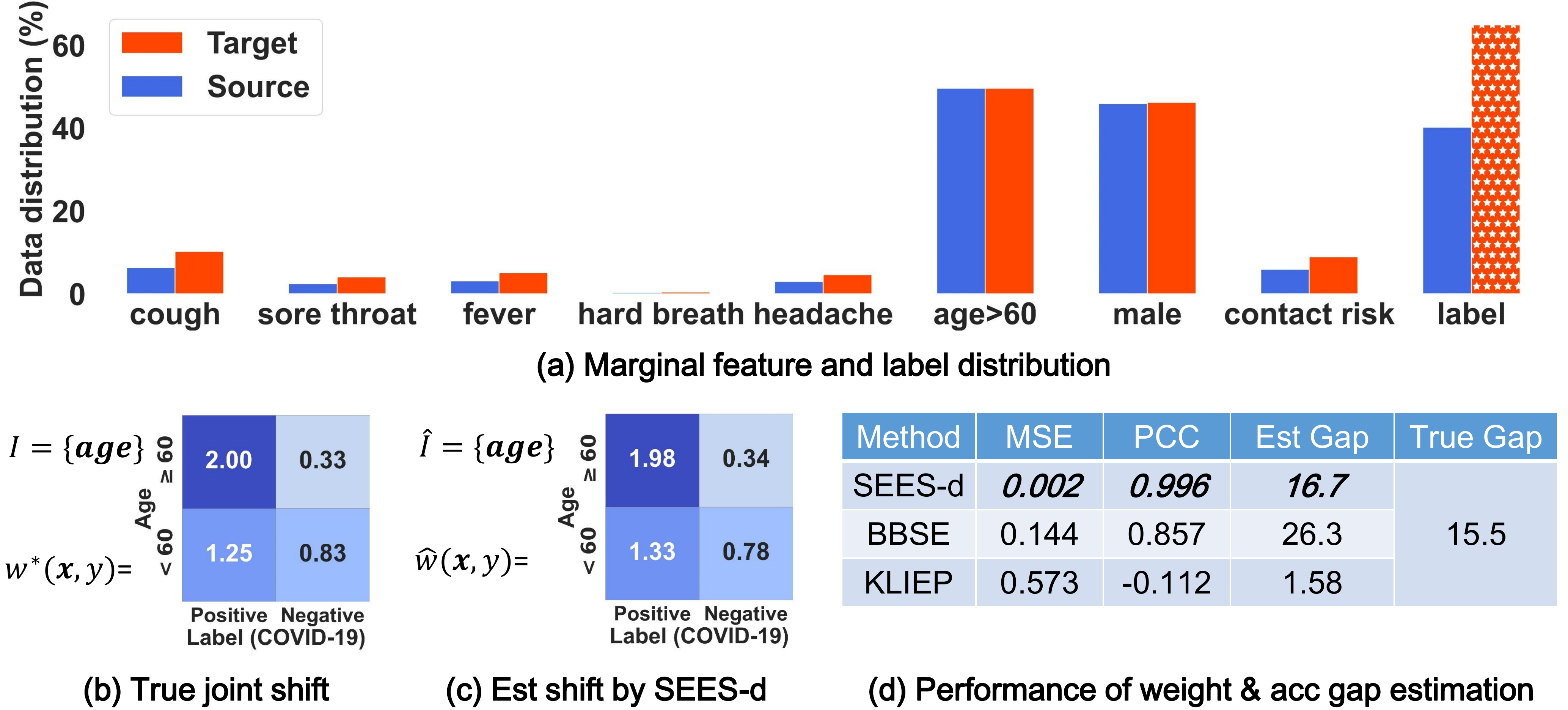}
    \caption{A case study on the  COVID-19 dataset. (a) The marginal distribution of labels and all features. The label on the target domain (the last red bar) is not observable.
    (b) The actual joint shift between source and target data. (c) The mean square error (MSE), pearson correlation coefficient (PCC) between learned and true weights, and the the estimated  accuracy gap. 
    Overall,   \systemnameShiftAtt{}-d significantly improves estimation performance over existing methods.}
    \label{fig:SparseShift:casestudyCOVID}
\end{figure}

Figure \ref{fig:SparseShift:casestudyCOVID} summarizes this case study.
First note that identifying which feature is shifted is not obvious. 
As shown in Figure \ref{fig:SparseShift:casestudyCOVID}(a), the marginal distribution of most features except age and gender has changed from the source to the target. 
The actual joint shift (Figure \ref{fig:SparseShift:casestudyCOVID}(b)) is, on the other hand, due to age group and the labels. 
Identifying shifted age is  challenging as the label on the target (last red bar in Figure \ref{fig:SparseShift:casestudyCOVID}(a)) also shifts but cannot be observed.
On the other hand,  \systemnameShiftAtt{}-d correctly identifies the shifted feature age, and produces a weight function close to the true weight   
(Figure \ref{fig:SparseShift:casestudyCOVID}(b) and (c)).
This is primarily because \systemnameShiftAtt{}-d explicitly exploits the joint shift modeled by \systemnameISS{}. 
In fact, \systemnameShiftAtt{}-d's performance is significantly better than existing methods. As shown in Figure \ref{fig:SparseShift:casestudyCOVID}(d), the mean square error (MSE) between the true weights and learned weights is only 0.002 when adopting  \systemnameShiftAtt{}-d, but 0.144 and 0.573 when using BBSE and KLIEP, respectively.  The Pearson correlation  coefficient (PCC)   between the true weights and weights learned by \systemnameShiftAtt{}-d is 0.996, indicating a strong correlation. The weight estimation performance directly affects how precise the estimated accuracy gap is.
The estimated gap $\hat{\Delta}$ of \systemnameShiftAtt{}-d is 16.7\%, which is close to the true gap (15.5\%).
By contrast, BBSE tends to overestimate the gap (26.3\$) while KLIEP underestimates it (1.58\%). 

\begin{figure} \centering
\begin{subfigure}[data shift]{\label{fig:sample_a}\includegraphics[width=0.225\linewidth]{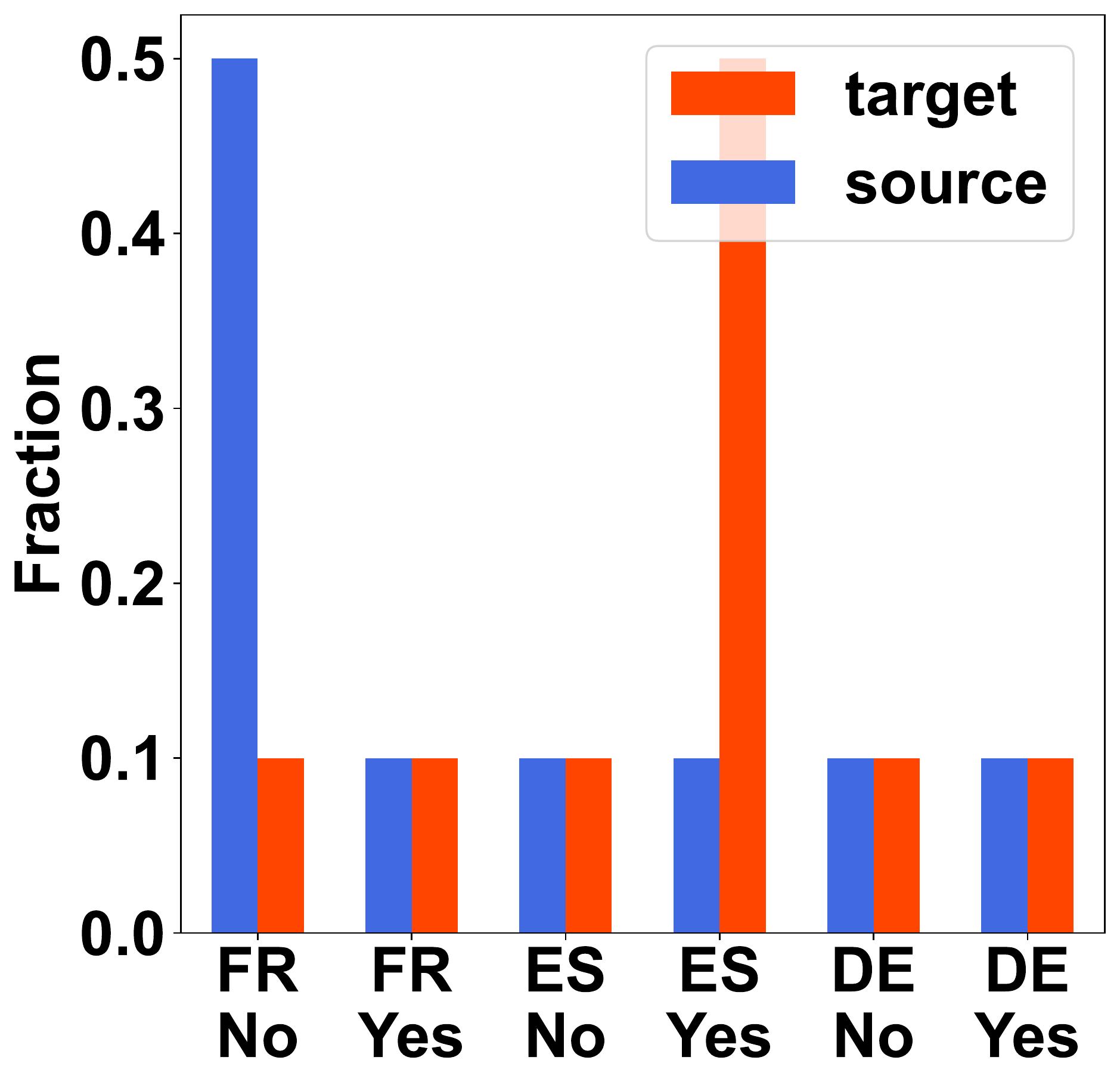}}
\end{subfigure}
\begin{subfigure}[estimated accuracy]{\label{fig:sample_b}\includegraphics[width=0.24\linewidth]{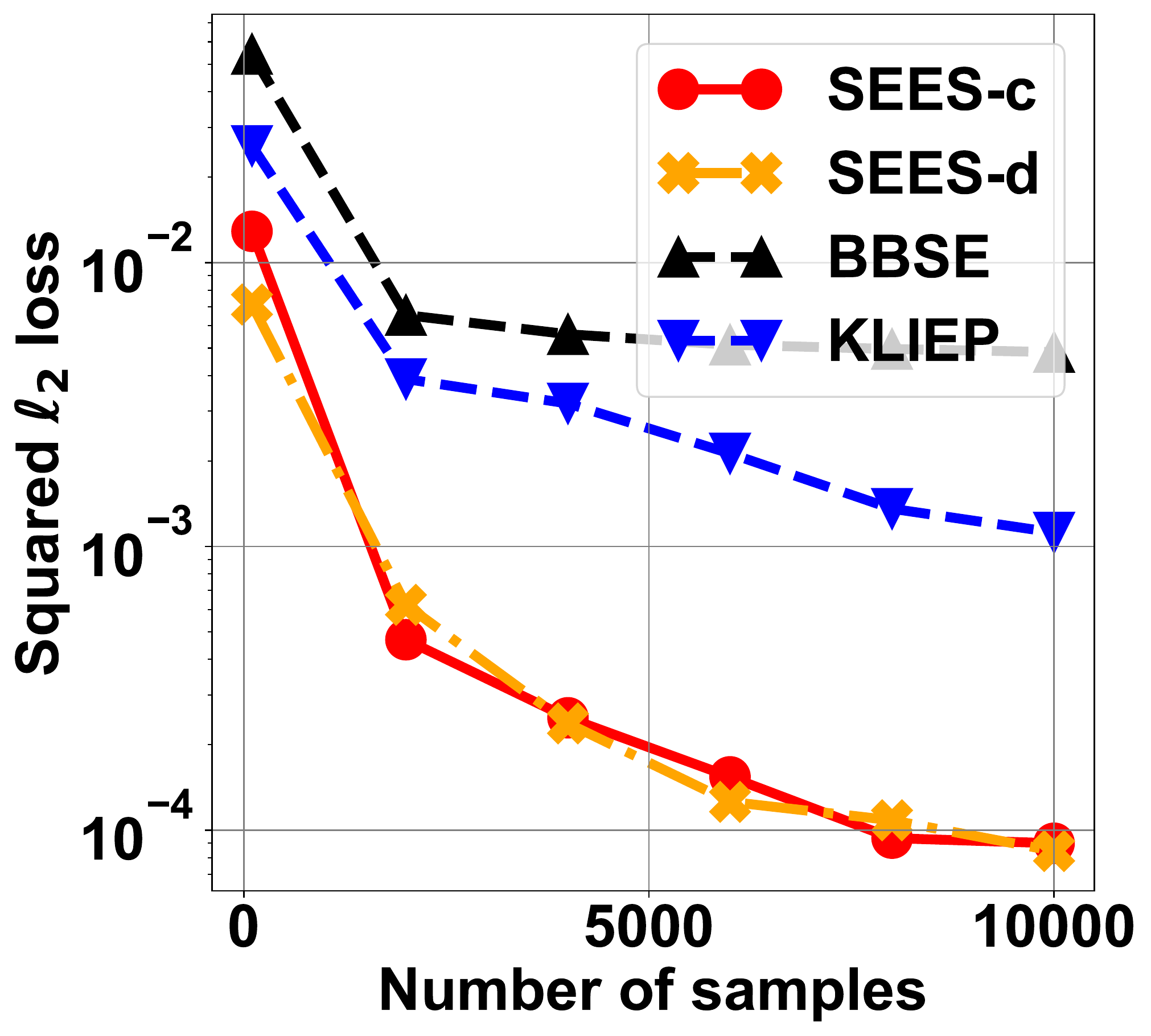}}
\end{subfigure}
\begin{subfigure}[estimated weights]{\label{fig:sample_c}\includegraphics[width=0.24\linewidth]{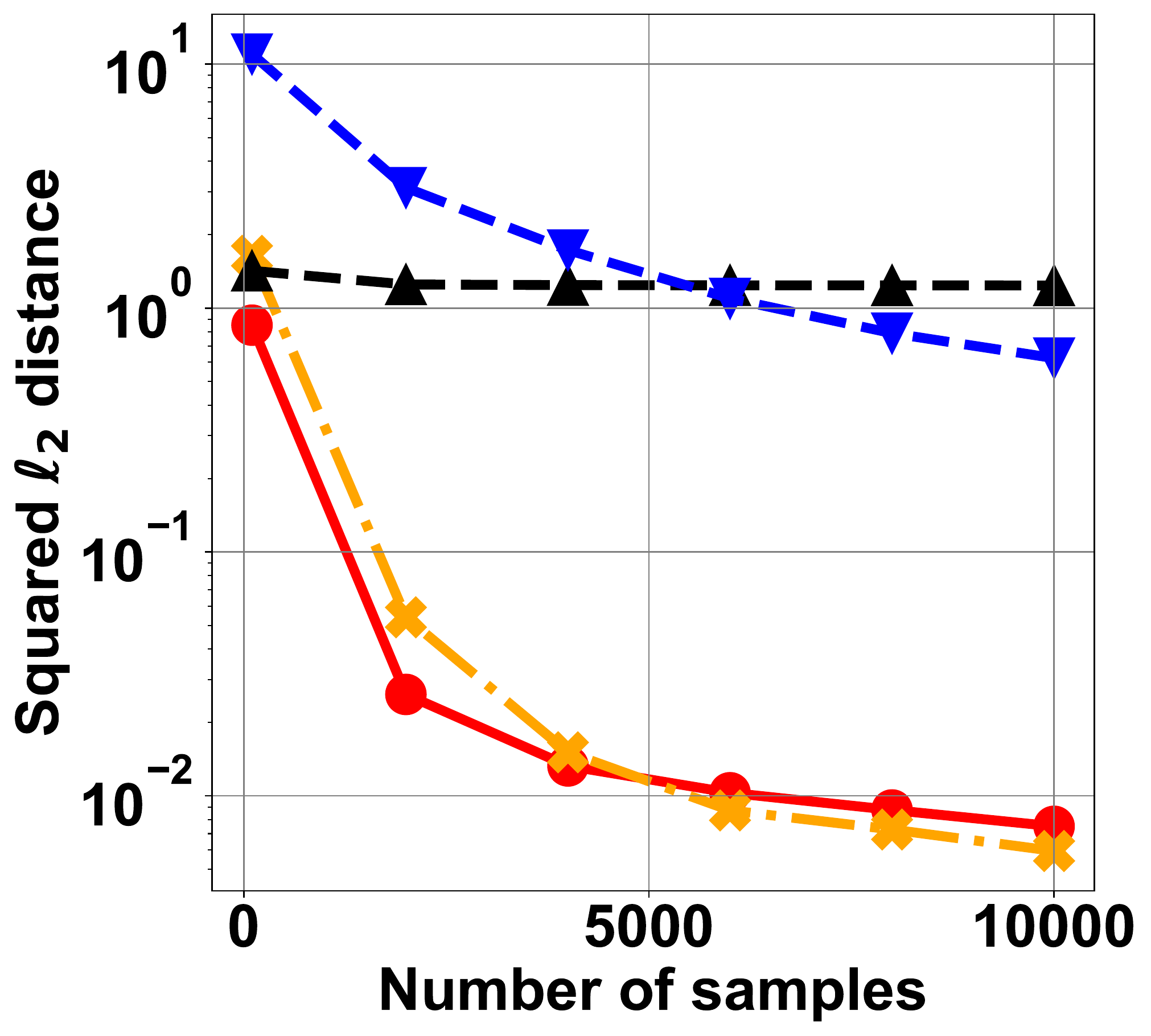}}
\end{subfigure}
\begin{subfigure}[shift discovery rate]{\label{fig:sample_d}\includegraphics[width=0.23\linewidth]{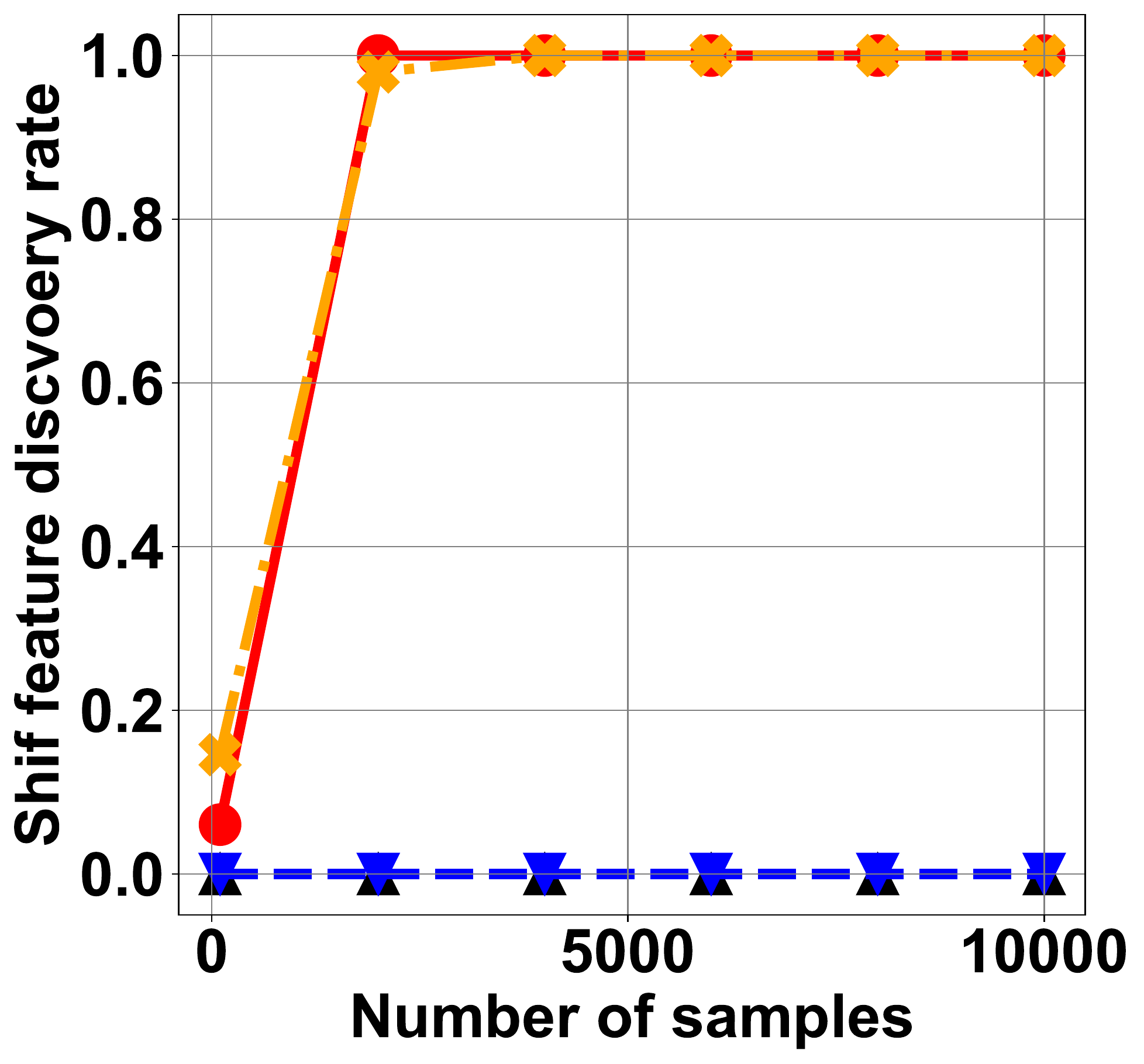}}
\end{subfigure}

\begin{subfigure}[data shift]{\label{fig:sample_e}\includegraphics[width=0.225\linewidth]{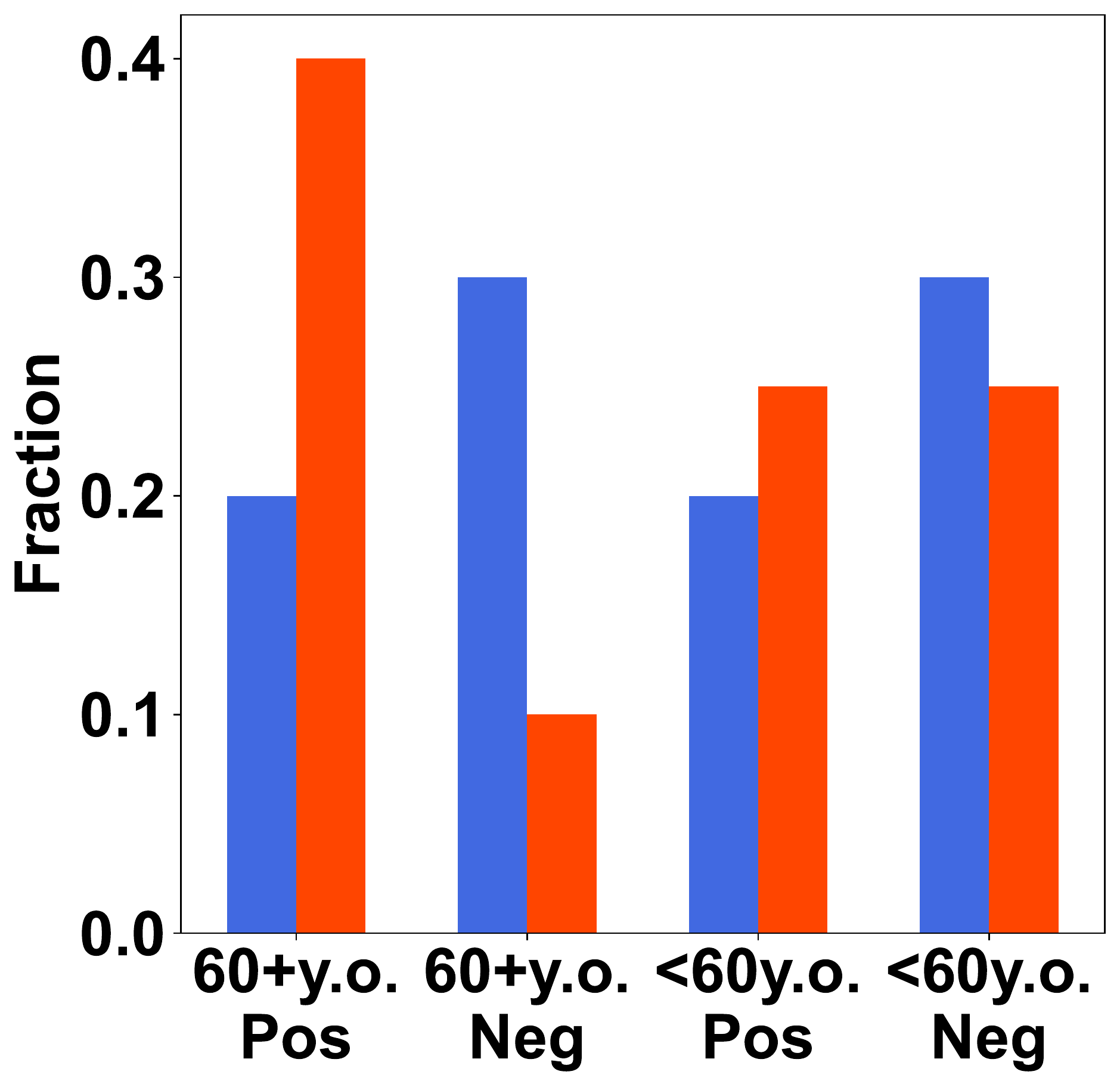}}
\end{subfigure}
\begin{subfigure}[estimated accuracy]{\label{fig:sample_f}\includegraphics[width=0.24\linewidth]{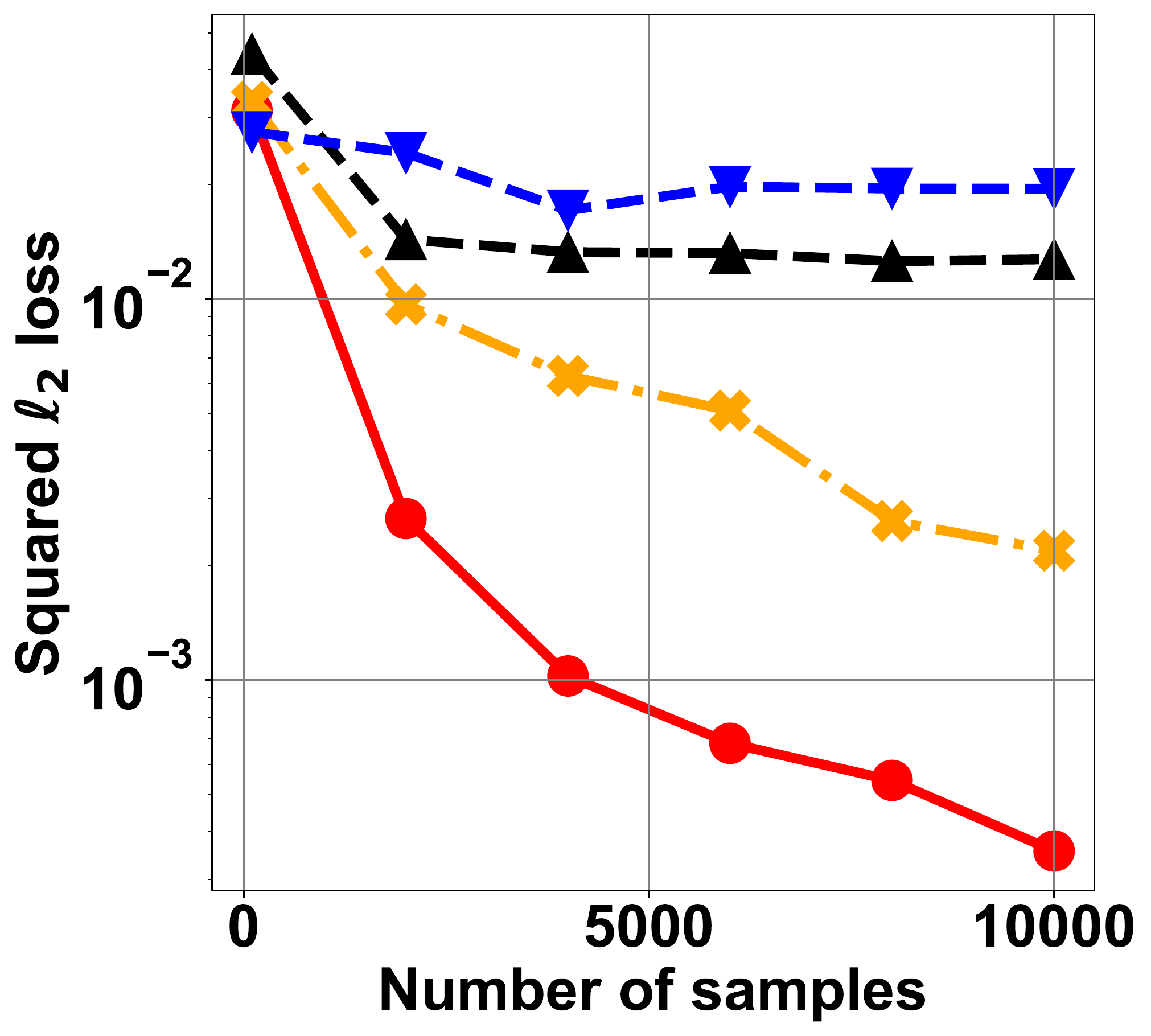}}
\end{subfigure}
\begin{subfigure}[estimated weights]{\label{fig:sample_g}\includegraphics[width=0.24\linewidth]{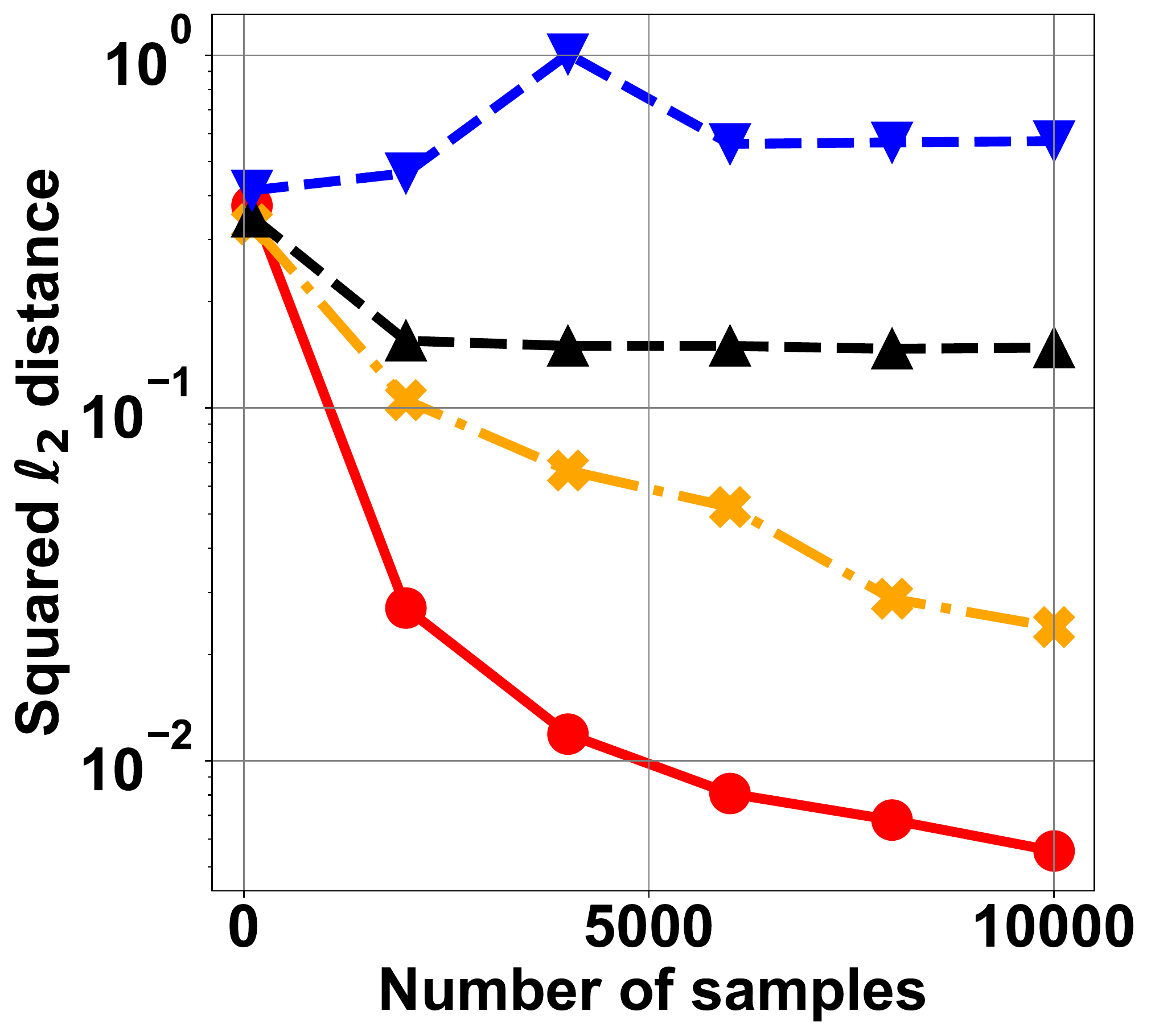}}
\end{subfigure}
\begin{subfigure}[shift discovery rate]{\label{fig:sample_h}\includegraphics[width=0.23\linewidth]{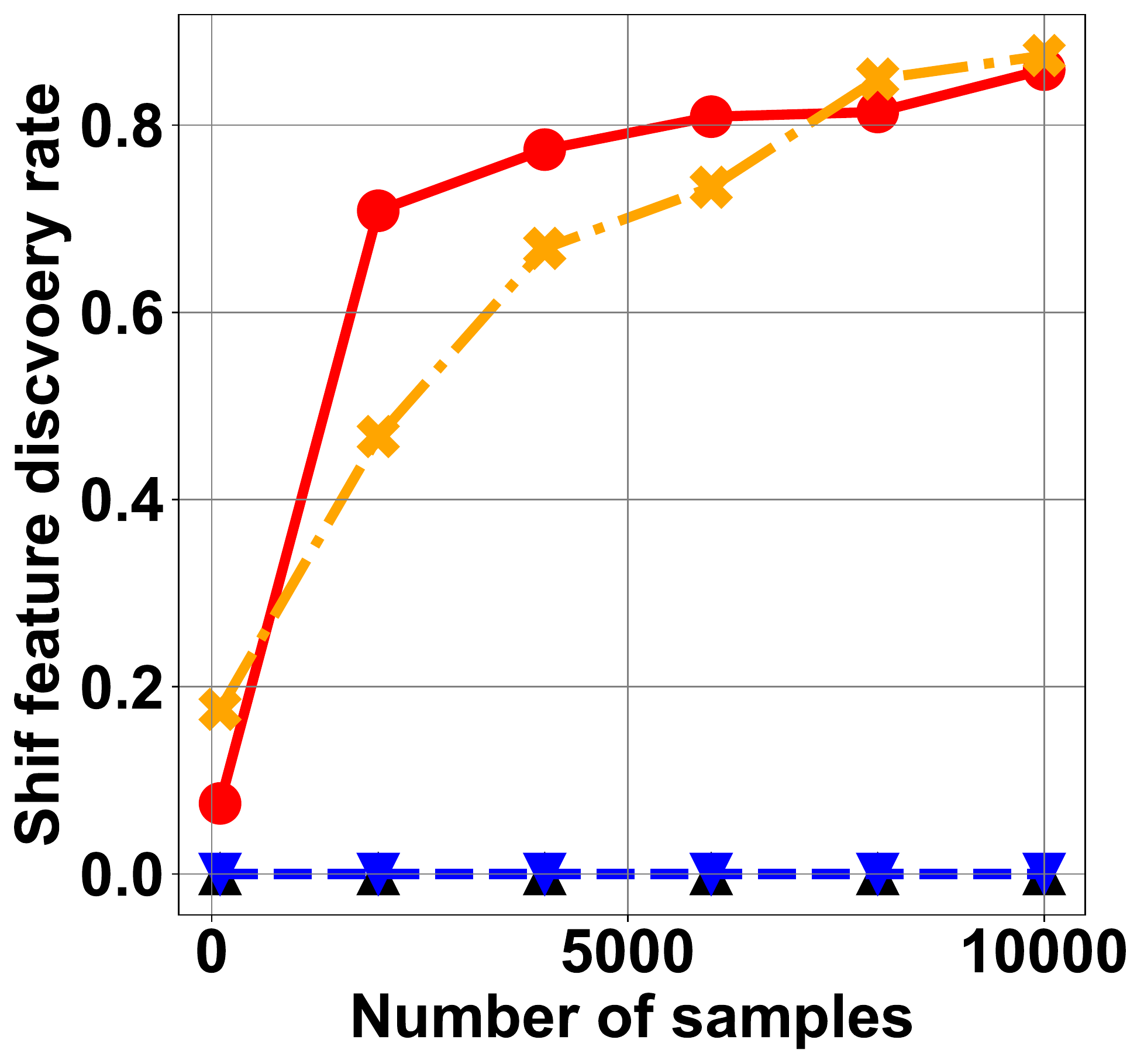}}
\end{subfigure}

\begin{subfigure}[data shift]{\label{fig:sample_i}\includegraphics[width=0.23\linewidth]{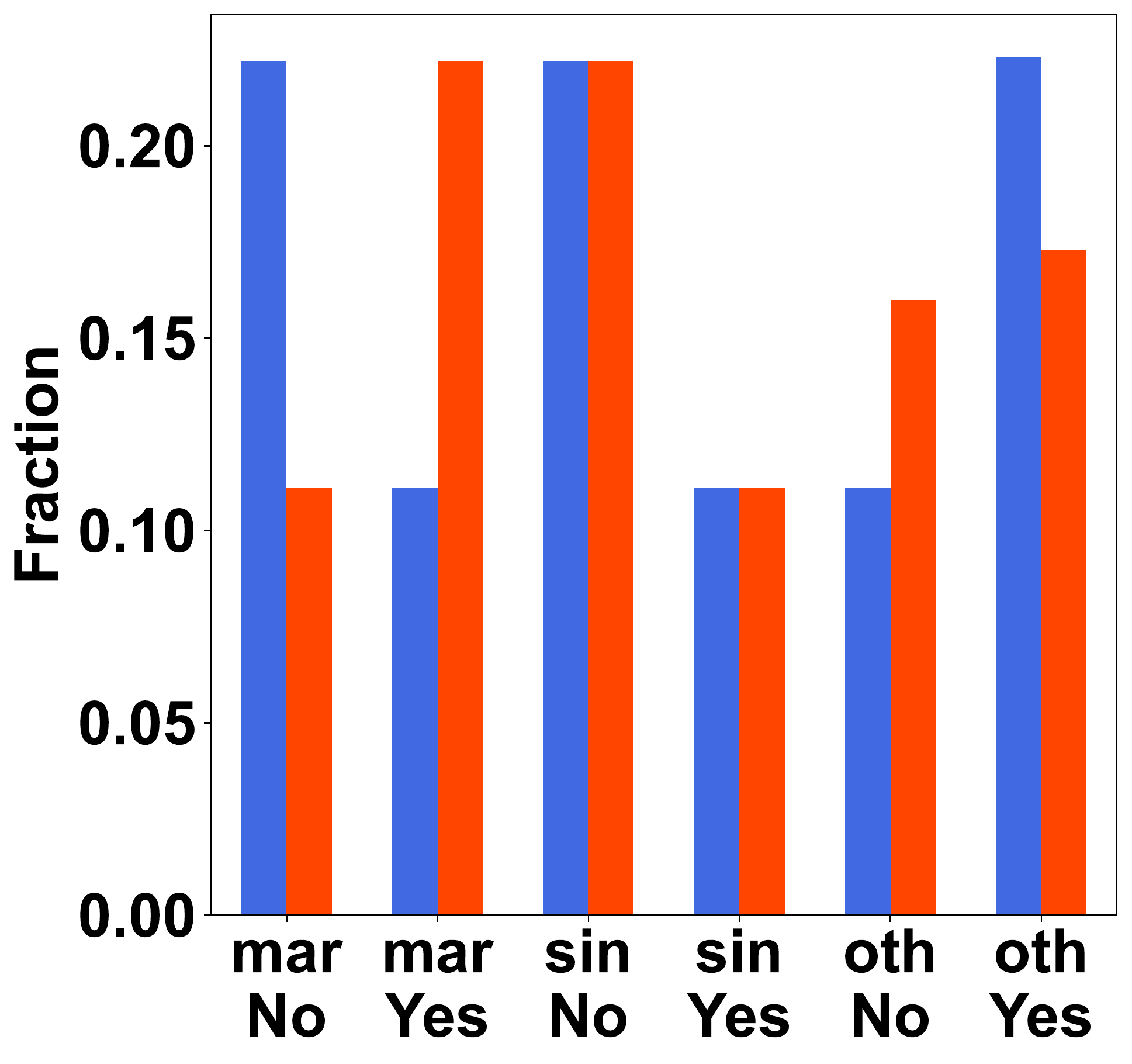}}
\end{subfigure}
\begin{subfigure}[estimated accuracy]{\label{fig:sample_j}\includegraphics[width=0.24\linewidth]{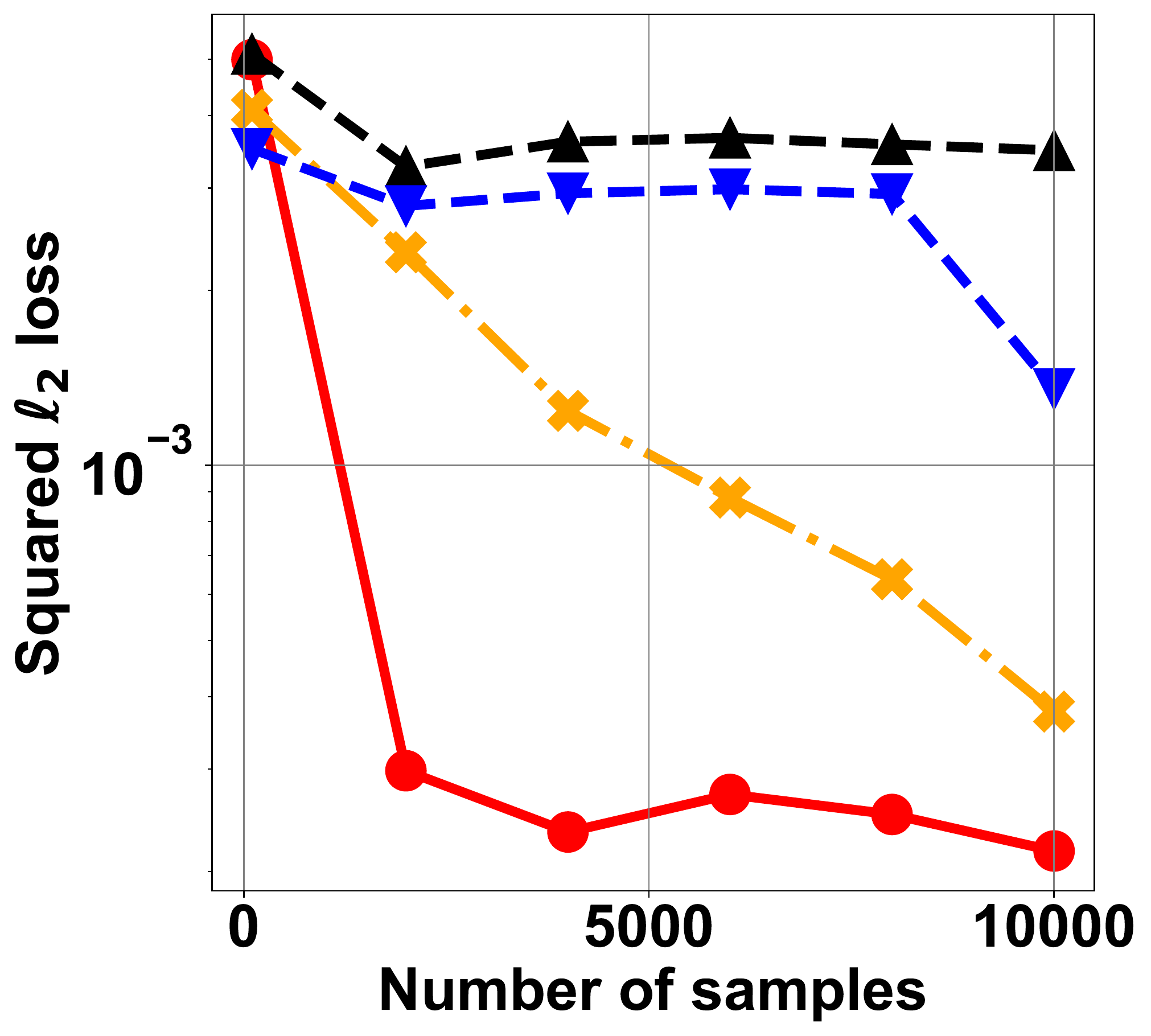}}
\end{subfigure}
\begin{subfigure}[estimated weights]{\label{fig:sample_k}\includegraphics[width=0.24\linewidth]{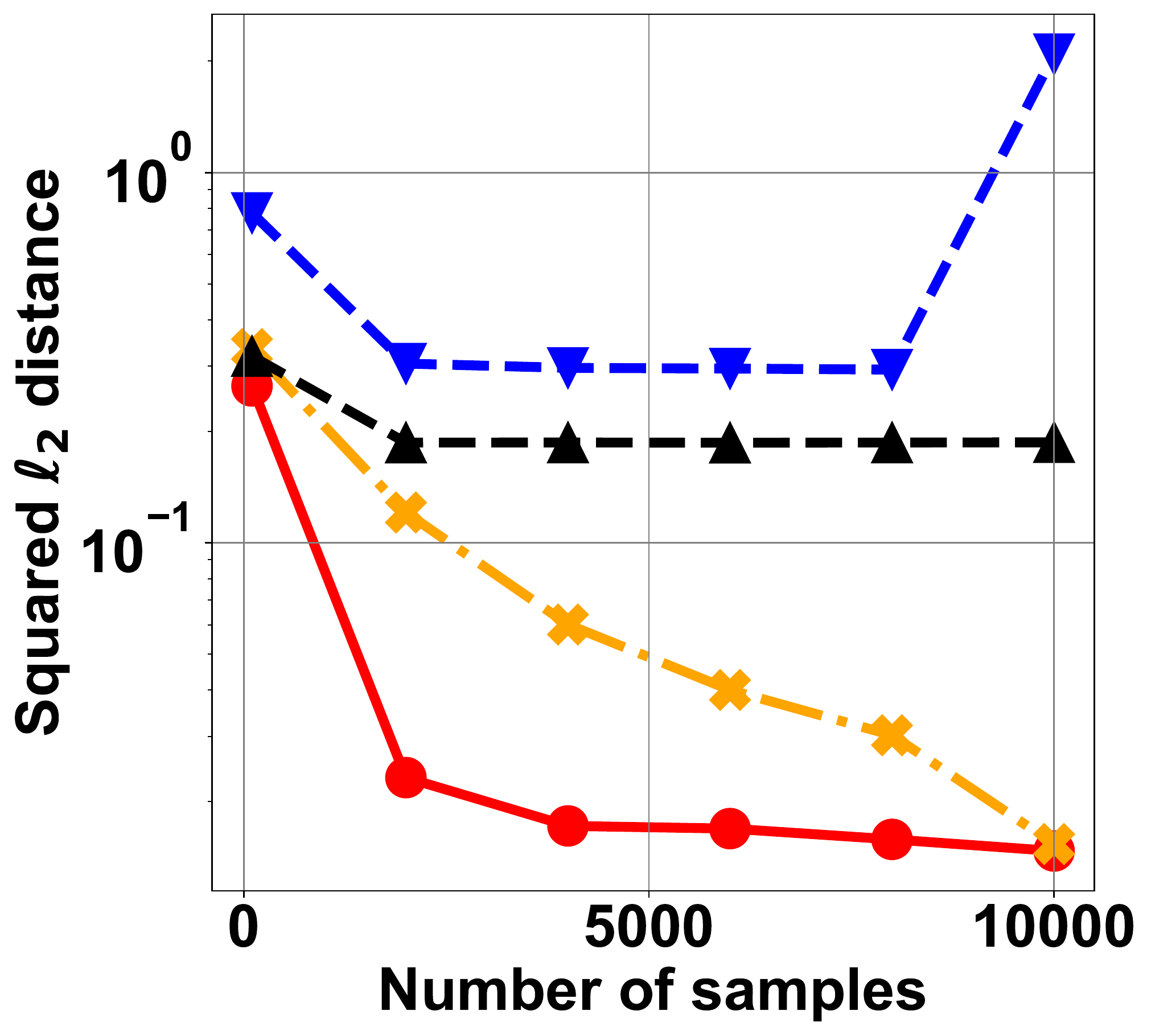}}
\end{subfigure}
\begin{subfigure}[shift discovery rate]{\label{fig:sample_l}\includegraphics[width=0.23\linewidth]{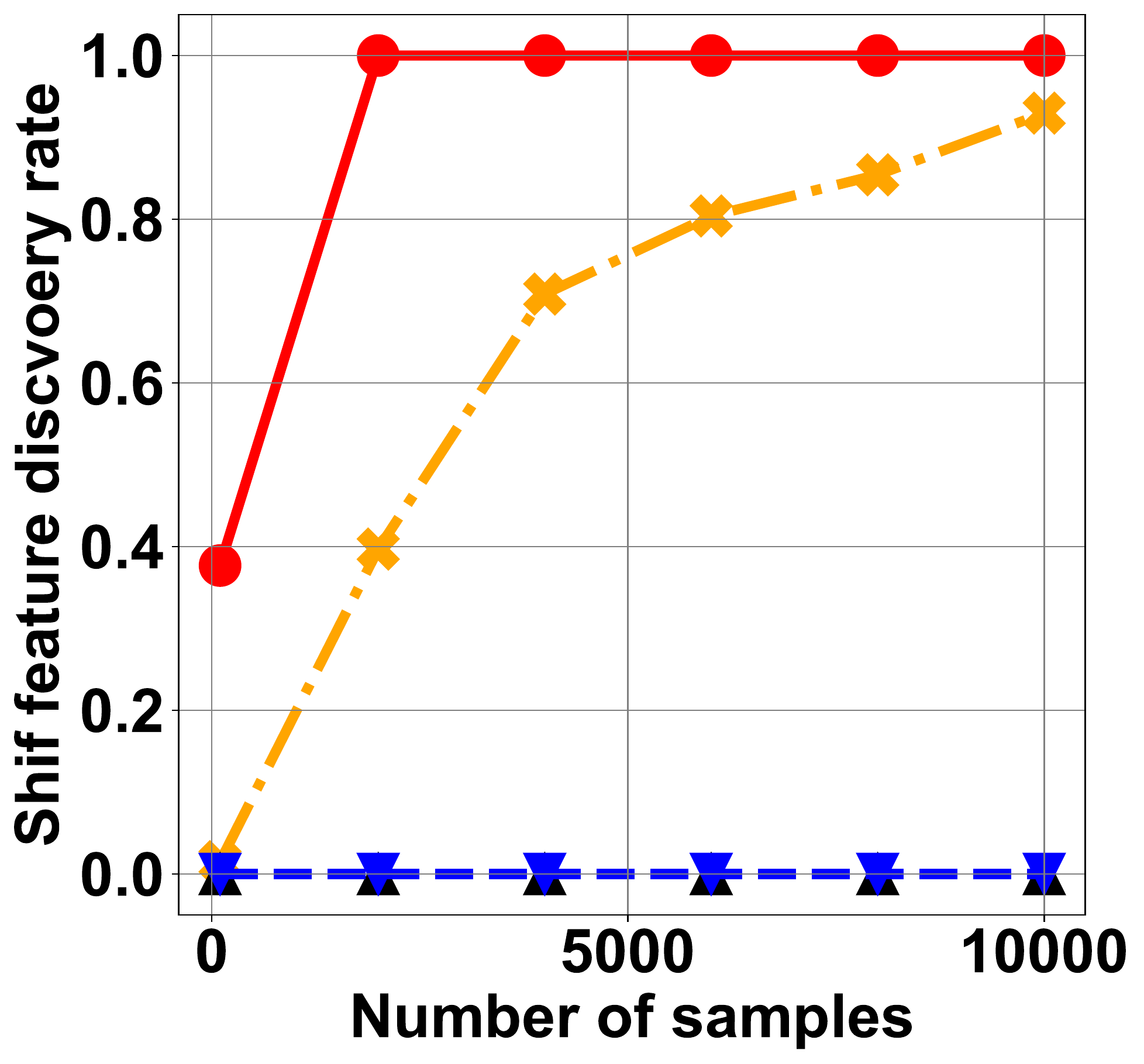}}
\end{subfigure}

\vspace{-0mm}
	\caption{Trade-offs between shift estimation and sample size. We vary the shifted features (first column), and measure performance of the estimated accuracy shift (second column), estimated weights (third column), and how often the true shifted features are discovered (last column). The first, second, and third row corresponds to dataset BANKCHURN, COVID-19, and CREDIT, respectively. Overall, both \systemnameShiftAtt{}-c and \systemnameShiftAtt{}-d consistently outperforms existing estimation approaches.  }\label{fig:SparseShift:tradeoffs}
\end{figure}

\paragraph{Trade-offs between estimation performance and sample size.} Next, we study the trade-offs between estimation performance and the number of available samples. 
For simplicity, we simulate various sparse joint shifts with $s=1$ via (i) specifying marginal distribution of  the shifted feature and labels first,   and (ii) then drawing random samples from the original dataset  conditional on values of specified labels and shifted feature. 
Same number of samples are allocated to both source and target datasets. 
Figure \ref{fig:SparseShift:tradeoffs} shows the simulated data shift (column 1), the squared $\ell_2$ loss of accuracy estimation (column 2), weight estimation (column 3), and the shifted feature discovery rate (column 4) for three datasets. 
Overall, we observe that the estimation error of  \systemnameShiftAtt{}-c and \systemnameShiftAtt{}-d diminishes as the number of samples increases, while that of  BBSE and KLIEP is almost flat. 

\begin{table}[htbp]
  \centering
  \scriptsize
  \caption{Root mean square error of estimated accuracy gap (\%) under real shifts for a gradient boosting model.
  The numbers are averaged over all source-target pairs in each dataset.
  Results for other models (e.g., a neural network) can be found in the Appendix.
  For each dataset and ML model, \systemnameShiftAtt{} provides significant estimation error reduction over baselines.
  }
    \begin{tabular}{|c|c|c|c|c|c|c|c|c|c|c|c|}
    \hline
    \multicolumn{4}{|c|}{EMPLOY}  & \multicolumn{4}{c|}{INCOME}   & \multicolumn{4}{c|}{INSURANCE} \bigstrut\\
    \hline
    SEES-c & SEES-d & BBSE  & KLIEP & SEES-c & SEES-d & BBSE  & KLIEP & SEES-c & SEES-d & BBSE  & KLIEP \bigstrut\\
    \hline
    \textbf{2.90 } & 3.00  & 5.20  & 5.20  & \textbf{1.90 } & 2.40  & 3.00  & 3.40  & \textbf{1.70 } & 2.20  & 2.10  & 5.00  \bigstrut\\
    \hline
    \end{tabular}%
  \label{tab:SparseShift:realshift}%
\end{table}%

\paragraph{Effects of shift sparsity.}
We have focused on $1$-\systemnameISS{} for simplicity, but how the sparsity $s$ affects the estimation performance remains  unknown. To answer this, we fix the number of samples to be 10,000, and then measure the performance of all compared methods for joint distribution of label and different number of features.
Overall, we observe that 
the estimation error often grows as more features shift jointly with the labels. 
This is because more shifted features often implies higher complexity in the distribution shifts and thus more parameters to estimate. 
It is worthy noting that
 $0$-\systemnameISS{} degenerates to label shift, and the performance of \systemnameShiftAtt{}-c and \systemnameShiftAtt{}-d is slightly worse than BBSE (designed for label shift) under $0$-\systemnameISS{}.
This is a trade-off for covering more general distribution shifts. More details can be found in the appendix.


\paragraph{Accuracy gap estimation on real  shifts.}
Finally we validate the effectiveness of \systemnameShiftAtt{} on accuracy estimation with real world shifts. 
EMPLOY and INCOME are  partitioned by geography (states) and INSURANCE is partitioned by time (year). 
For each partition pair, we train a gradient boosting model on one and estimate its performance on the other. Table \ref{tab:SparseShift:realshift}  shows the estimation error averaged over all partition pairs for each dataset.
\systemnameShiftAtt{} consistently outperforms BBSE and KLIEP, and reduces the estimation error by  up to 66\% (1-1.7/5.0).
We also evaluate other models (including a neural network) and observe similar results. 
More results can be  found in  the Appendix. 

\eat{
and thus we train the ML model on one state and estimate its performance on another. 
Insurance is divided based on which year the data is collected (2014, 2016, 2018), and thus we train the model for one year's data and estimate its performance on another. The evaluation was performed for three different models:  a gradient boosting, a neural network, and a decision tree. We leave more details to the Appendix due to space constraints. 
Table \ref{tab:SparseShift:realshift} lists the average accuracy estimation error for each dataset.
Overall, \systemnameShiftAtt{} consistently outperforms BBSE and KLIEP. This is because the joint distribution is not fixed across locations or years, and thus assuming pure label or covariate shift leads to misleading performance estimation. 
}


\section{Conclusion}\label{Sec:SparseShift:Conclusion}
In this paper,  we propose Sparse Joint Shift (\systemnameISS), a new distribution shift model that accounts for both label and covariate shifts.
We show how \systemnameISS{} unifies and generalizes  existing distribution shift models and remains identifiable under reasonable assumptions. We develop  \systemnameShiftAtt{}, an algorithmic framework for unsupervised model performance estimation and explanation under \systemnameISS{}.
Both theoretical analysis and empirical study validates the effectiveness of  \systemnameShiftAtt{}. Our work contributes to making ML more reliable when data can change.
A natural next step is how to improve estimation performance under \systemnameISS{} when a small number of target labels can be queried.
This paper focuses on model performance estimation and explanation. 
 Developing ML models robust to different \systemnameISS{} is also an open question. 

\newpage
{
\small

\bibliography{MLService}

\begin{thebibliography}{10}

\bibitem{dataset_bankchurn}
{The BANKCHURN dataset}.
\newblock
  \url{https://www.kaggle.com/code/kmalit/bank-customer-churn-prediction/notebook}.
\newblock [Accessed 2022].

\bibitem{dataset_covid19}
{The COVID-19 government dataset}.
\newblock \url{https://data.gov.il/dataset/covid-19}.
\newblock [Accessed 2022].

\bibitem{labelshift_mle_alexandari2020maximum}
Amr Alexandari, Anshul Kundaje, and Avanti Shrikumar.
\newblock Maximum likelihood with bias-corrected calibration is hard-to-beat at
  label shift adaptation.
\newblock In {\em International Conference on Machine Learning}, pages
  222--232. PMLR, 2020.

\bibitem{covariate_application_spam_bickel2006dirichlet}
Steffen Bickel and Tobias Scheffer.
\newblock Dirichlet-enhanced spam filtering based on biased samples.
\newblock {\em Advances in neural information processing systems}, 19, 2006.

\bibitem{LabelShift_epidemiologic1966}
Alfred~A Buck and John~J Gart.
\newblock Comparison of a screening test and a reference test in epidemiologic
  studies: I. indices of agreement and their relation to prevalence.
\newblock {\em American Journal of Epidemiology}, 83(3):586--592, 1966.

\bibitem{ModelEval_chen2021detecting}
Jiefeng Chen, Frederick Liu, Besim Avci, Xi~Wu, Yingyu Liang, and Somesh Jha.
\newblock Detecting errors and estimating accuracy on unlabeled data with
  self-training ensembles.
\newblock {\em Advances in Neural Information Processing Systems}, 34, 2021.

\bibitem{chen2021did}
Lingjiao Chen, Matei Zaharia, and James Zou.
\newblock How did the model change? efficiently assessing machine learning api
  shifts.
\newblock In {\em International Conference on Learning Representations}, 2021.

\bibitem{ModelEval_mandoline_chen2021}
Mayee Chen, Karan Goel, Nimit~S Sohoni, Fait Poms, Kayvon Fatahalian, and
  Christopher R{\'e}.
\newblock Mandoline: Model evaluation under distribution shift.
\newblock In {\em International Conference on Machine Learning}, pages
  1617--1629. PMLR, 2021.

\bibitem{ModelEval_MIT_chuang2020estimating}
Ching-Yao Chuang, Antonio Torralba, and Stefanie Jegelka.
\newblock Estimating generalization under distribution shifts via
  domain-invariant representations.
\newblock {\em arXiv preprint arXiv:2007.03511}, 2020.

\bibitem{ModelEval_rotation_deng2021does}
Weijian Deng, Stephen Gould, and Liang Zheng.
\newblock What does rotation prediction tell us about classifier accuracy under
  varying testing environments?
\newblock In {\em International Conference on Machine Learning}, pages
  2579--2589. PMLR, 2021.

\bibitem{ModelEval_featurestats_deng2021labels}
Weijian Deng and Liang Zheng.
\newblock Are labels always necessary for classifier accuracy evaluation?
\newblock In {\em Proceedings of the IEEE/CVF Conference on Computer Vision and
  Pattern Recognition}, pages 15069--15078, 2021.

\bibitem{dataset_census_ding2021retiring}
Frances Ding, Moritz Hardt, John Miller, and Ludwig Schmidt.
\newblock Retiring adult: New datasets for fair machine learning.
\newblock {\em Advances in Neural Information Processing Systems}, 34, 2021.

\bibitem{ModelEval_donmez2010unsupervised}
Pinar Donmez, Guy Lebanon, and Krishnakumar Balasubramanian.
\newblock Unsupervised supervised learning i: Estimating classification and
  regression errors without labels.
\newblock {\em Journal of Machine Learning Research}, 11(4), 2010.

\bibitem{forman2008quantifying}
George Forman.
\newblock Quantifying counts and costs via classification.
\newblock {\em Data Mining and Knowledge Discovery}, 17(2):164--206, 2008.

\bibitem{covariate_gretton2009}
Arthur Gretton, Alex Smola, Jiayuan Huang, Marcel Schmittfull, Karsten
  Borgwardt, and Bernhard Sch{\"o}lkopf.
\newblock Covariate shift by kernel mean matching.
\newblock {\em Dataset shift in machine learning}, 3(4):5, 2009.

\bibitem{DemographicEffects2019}
Patrick Grother, Mei Ngan, and Kayee Hanaoka.
\newblock Face recognition vendor test part 3: Demographic effects, 2019-12-19
  2019.

\bibitem{modeleval_confidence_guillory2021predicting}
Devin Guillory, Vaishaal Shankar, Sayna Ebrahimi, Trevor Darrell, and Ludwig
  Schmidt.
\newblock Predicting with confidence on unseen distributions.
\newblock In {\em Proceedings of the IEEE/CVF International Conference on
  Computer Vision}, pages 1134--1144, 2021.

\bibitem{guillory2021predicting}
Devin Guillory, Vaishaal Shankar, Sayna Ebrahimi, Trevor Darrell, and Ludwig
  Schmidt.
\newblock Predicting with confidence on unseen distributions.
\newblock In {\em Proceedings of the IEEE/CVF International Conference on
  Computer Vision}, pages 1134--1144, 2021.

\bibitem{covaraite_applications_humanactivity_hachiya2012importance}
Hirotaka Hachiya, Masashi Sugiyama, and Naonori Ueda.
\newblock Importance-weighted least-squares probabilistic classifier for
  covariate shift adaptation with application to human activity recognition.
\newblock {\em Neurocomputing}, 80:93--101, 2012.

\bibitem{jiang2019fantastic}
Yiding Jiang, Behnam Neyshabur, Hossein Mobahi, Dilip Krishnan, and Samy
  Bengio.
\newblock Fantastic generalization measures and where to find them.
\newblock {\em arXiv preprint arXiv:1912.02178}, 2019.

\bibitem{covariate_emotion_jirayucharoensak2014eeg}
Suwicha Jirayucharoensak, Setha Pan-Ngum, and Pasin Israsena.
\newblock Eeg-based emotion recognition using deep learning network with
  principal component based covariate shift adaptation.
\newblock {\em The Scientific World Journal}, 2014, 2014.

\bibitem{ImbalanceData19}
Harsurinder Kaur, Husanbir~Singh Pannu, and Avleen~Kaur Malhi.
\newblock A systematic review on imbalanced data challenges in machine
  learning: Applications and solutions.
\newblock {\em {ACM} Comput. Surv.}, 52(4):79:1--79:36, 2019.

\bibitem{Wild_Data21}
Pang~Wei Koh, Shiori Sagawa, Henrik Marklund, Sang~Michael Xie, Marvin Zhang,
  Akshay Balsubramani, Weihua Hu, Michihiro Yasunaga, Richard~Lanas Phillips,
  Irena Gao, et~al.
\newblock Wilds: A benchmark of in-the-wild distribution shifts.
\newblock In {\em International Conference on Machine Learning}, pages
  5637--5664. PMLR, 2021.

\bibitem{BBSE18}
Zachary Lipton, Yu-Xiang Wang, and Alexander Smola.
\newblock Detecting and correcting for label shift with black box predictors.
\newblock In {\em International conference on machine learning}, pages
  3122--3130. PMLR, 2018.

\bibitem{LabelShift_econ1977}
Charles~F Manski and Steven~R Lerman.
\newblock The estimation of choice probabilities from choice based samples.
\newblock {\em Econometrica: Journal of the Econometric Society}, pages
  1977--1988, 1977.

\bibitem{covariateshift_unify_moreno2012unifying}
Jose~G Moreno-Torres, Troy Raeder, Roc{\'\i}o Alaiz-Rodr{\'\i}guez, Nitesh~V
  Chawla, and Francisco Herrera.
\newblock A unifying view on dataset shift in classification.
\newblock {\em Pattern recognition}, 45(1):521--530, 2012.

\bibitem{park2020calibrated}
Sangdon Park, Osbert Bastani, James Weimer, and Insup Lee.
\newblock Calibrated prediction with covariate shift via unsupervised domain
  adaptation.
\newblock In {\em International Conference on Artificial Intelligence and
  Statistics}, pages 3219--3229. PMLR, 2020.

\bibitem{EmpiricalShiftStudy19}
Stephan Rabanser, Stephan G{\"u}nnemann, and Zachary Lipton.
\newblock Failing loudly: An empirical study of methods for detecting dataset
  shift.
\newblock {\em Advances in Neural Information Processing Systems}, 32, 2019.

\bibitem{recht2019imagenet}
Benjamin Recht, Rebecca Roelofs, Ludwig Schmidt, and Vaishaal Shankar.
\newblock Do imagenet classifiers generalize to imagenet?
\newblock In {\em International Conference on Machine Learning}, pages
  5389--5400. PMLR, 2019.

\bibitem{covariate_fairness_rezaei2021}
Ashkan Rezaei, Anqi Liu, Omid Memarrast, and Brian~D Ziebart.
\newblock Robust fairness under covariate shift.
\newblock In {\em Proceedings of the AAAI Conference on Artificial
  Intelligence}, volume~35, pages 9419--9427, 2021.

\bibitem{covariate_robustness_schneider2020}
Steffen Schneider, Evgenia Rusak, Luisa Eck, Oliver Bringmann, Wieland Brendel,
  and Matthias Bethge.
\newblock Improving robustness against common corruptions by covariate shift
  adaptation.
\newblock {\em Advances in Neural Information Processing Systems},
  33:11539--11551, 2020.

\bibitem{covariate_firstpaper_shimodaira2000improving}
Hidetoshi Shimodaira.
\newblock Improving predictive inference under covariate shift by weighting the
  log-likelihood function.
\newblock {\em Journal of statistical planning and inference}, 90(2):227--244,
  2000.

\bibitem{covariate_sugiyama2007}
Masashi Sugiyama, Matthias Krauledat, and Klaus-Robert M{\"u}ller.
\newblock Covariate shift adaptation by importance weighted cross validation.
\newblock {\em Journal of Machine Learning Research}, 8(5), 2007.

\bibitem{KLIEP07}
Masashi Sugiyama, Shinichi Nakajima, Hisashi Kashima, Paul Buenau, and Motoaki
  Kawanabe.
\newblock Direct importance estimation with model selection and its application
  to covariate shift adaptation.
\newblock {\em Advances in neural information processing systems}, 20, 2007.

\bibitem{covariate_sugiyama2008direct}
Masashi Sugiyama, Taiji Suzuki, Shinichi Nakajima, Hisashi Kashima, Paul von
  B{\"u}nau, and Motoaki Kawanabe.
\newblock Direct importance estimation for covariate shift adaptation.
\newblock {\em Annals of the Institute of Statistical Mathematics},
  60(4):699--746, 2008.

\bibitem{taori2020measuring}
Rohan Taori, Achal Dave, Vaishaal Shankar, Nicholas Carlini, Benjamin Recht,
  and Ludwig Schmidt.
\newblock Measuring robustness to natural distribution shifts in image
  classification.
\newblock {\em Advances in Neural Information Processing Systems},
  33:18583--18599, 2020.

\bibitem{covariate_conformal_tibshirani2019}
Ryan~J Tibshirani, Rina Foygel~Barber, Emmanuel Candes, and Aaditya Ramdas.
\newblock Conformal prediction under covariate shift.
\newblock {\em Advances in neural information processing systems}, 32, 2019.

\bibitem{ModelEval_Lazy_welinder2013}
Peter Welinder, Max Welling, and Pietro Perona.
\newblock A lazy man's approach to benchmarking: Semisupervised classifier
  evaluation and recalibration.
\newblock In {\em Proceedings of the IEEE Conference on Computer Vision and
  Pattern Recognition}, pages 3262--3269, 2013.

\bibitem{LabelShift_Online2021}
Ruihan Wu, Chuan Guo, Yi~Su, and Kilian~Q Weinberger.
\newblock Online adaptation to label distribution shift.
\newblock {\em Advances in Neural Information Processing Systems}, 34, 2021.

\bibitem{covaraite_distcompare_yamada2011relative}
Makoto Yamada, Taiji Suzuki, Takafumi Kanamori, Hirotaka Hachiya, and Masashi
  Sugiyama.
\newblock Relative density-ratio estimation for robust distribution comparison.
\newblock {\em Advances in neural information processing systems}, 24, 2011.

\bibitem{dataset_credit_yeh2009comparisons}
I-Cheng Yeh and Che-hui Lien.
\newblock The comparisons of data mining techniques for the predictive accuracy
  of probability of default of credit card clients.
\newblock {\em Expert systems with applications}, 36(2):2473--2480, 2009.

\bibitem{LabelShift_active2021}
Eric Zhao, Anqi Liu, Animashree Anandkumar, and Yisong Yue.
\newblock Active learning under label shift.
\newblock In {\em International Conference on Artificial Intelligence and
  Statistics}, pages 3412--3420. PMLR, 2021.

\end{thebibliography}
\bibliographystyle{plain}
}

\section*{Checklist}

\begin{enumerate}

\item For all authors...
\begin{enumerate}
  \item Do the main claims made in the abstract and introduction accurately reflect the paper's contributions and scope?
    \answerYes{}{}
  \item Did you describe the limitations of your work?
    \answerYes{}. See Appendix \ref{sec:SparseShift:Discussions}.
  \item Did you discuss any potential negative societal impacts of your work?
    \answerYes{}. See Appendix \ref{sec:SparseShift:Discussions}.
  \item Have you read the ethics review guidelines and ensured that your paper conforms to them?
    \answerYes{}.
\end{enumerate}

\item If you are including theoretical results...
\begin{enumerate}
  \item Did you state the full set of assumptions of all theoretical results?
    \answerYes{}.
        \item Did you include complete proofs of all theoretical results?
    \answerYes.{See Appendix \ref{Sec:SparseShift:theorydtails}} 
\end{enumerate}

\item If you ran experiments...
\begin{enumerate}
  \item Did you include the code, data, and instructions needed to reproduce the main experimental results (either in the supplemental material or as a URL)?
    \answerYes{}
  \item Did you specify all the training details (e.g., data splits, hyperparameters, how they were chosen)?
    \answerYes{}. Please see Appendix \ref{sec:SparseShift:experimentdetails}.
        \item Did you report error bars (e.g., with respect to the random seed after running experiments multiple times)?
    \answerYes{}
        \item Did you include the total amount of compute and the type of resources used (e.g., type of GPUs, internal cluster, or cloud provider)?
    \answerYes{}. Please see Appendix \ref{sec:SparseShift:experimentdetails}.
\end{enumerate}

\item If you are using existing assets (e.g., code, data, models) or curating/releasing new assets...
\begin{enumerate}
  \item If your work uses existing assets, did you cite the creators?
    \answerYes{}.
  \item Did you mention the license of the assets?
    \answerNA{}
  \item Did you include any new assets either in the supplemental material or as a URL?
    \answerNA{}
  \item Did you discuss whether and how consent was obtained from people whose data you're using/curating?
    \answerNA{}
  \item Did you discuss whether the data you are using/curating contains personally identifiable information or offensive content?
    \answerNA{}
\end{enumerate}

\item If you used crowdsourcing or conducted research with human subjects...
\begin{enumerate}
  \item Did you include the full text of instructions given to participants and screenshots, if applicable?
    \answerNA{}.
  \item Did you describe any potential participant risks, with links to Institutional Review Board (IRB) approvals, if applicable?
    \answerNA{}.
  \item Did you include the estimated hourly wage paid to participants and the total amount spent on participant compensation?
    \answerNA{}.
\end{enumerate}

\end{enumerate}


\newpage

\appendix
\section{Broader Impact and Limitation Discussion }\label{sec:SparseShift:Discussions}
Monitoring,  estimating, and explaining performance of deployed ML models is a growing area with significant economic and social  impact.  
In this paper, we propose \systemnameISS{}, a new data distribution shift model to consider when both labels and features shift after model deployment. 
We show how \systemnameISS{} generalizes existing data shift models, and further propose \systemnameShiftAtt{}, a generic framework that efficiently explains and estimates an ML model's performance under \systemnameISS{}. 
This may serve as a monitoring tool to  help ML practitioners  recognize performance changes, discover potential fairness issues and take appropriate business  decisions (e.g., switching to other models or debugging the existing ones).  
One limitation in general is adaption to continuously changing data streams.
An online algorithm for performance estimation and explanation under \systemnameISS{} is in need to address this challenge. 
We will also open source our prototype of \systemnameShiftAtt{} serving as a resource for broad community to use. 

\section{Missing Proofs}\label{Sec:SparseShift:theorydtails}
We provide all missing proofs in this section.

\subsection{Proof of Theorem \ref{thm:SparseShift:identifiable}}
\begin{proof}
We prove the statement by contradiction. 
Suppose not. That is to say, there exists a distribution $p_a(\X,\Y)$, such that $p_a(\X)=\dQ(\X)$, and there exists some set $\mathcal{J}\subset [d], |\mathcal{J}|\leq \sparse$, such that $p_a(\X_{\mathcal{J}^c}|\X_\mathcal{J},\Y) = \dP(\X_{\mathcal{J}^c}|\X_{\mathcal{J}},\Y)$ and $p_a(\X,\Y)\not=\dQ(\X,\Y)$.
Let $w_a(\X,\Y)\triangleq \frac{p_a(\X,\Y)}{\dP(\X,\Y)}$ denote the ratio between this alternative distribution $p_a$ and the source distribution.
Recall that $\w^*(\X,\Y)=\frac{\dQ(\X,\Y)}{\dP(\X,\Y)}$.
To show the contradiction, we simply need to show $p_a(\X,\Y)=\dQ(\X,\Y)$, which is equivalent to show $w_a(\X,\Y) = w^*(\X,\Y)$.

By the $\sparse$-\systemnameISS{} assumption, $ \dQ(\X_\IndSet^c|\X_\IndSet,\Y)= \dP(\X_\IndSet^c|\X_\IndSet,\Y)$, we have 
\begin{equation*}
    \w^*(\X,\Y) = \frac{\dQ(\X,\Y)}{\dP(\X,\Y)} = \frac{\dQ(\X_\IndSet,\Y) \dQ(\X_\IndSet^c|\X_\IndSet,\Y)}{\dP(\X_\IndSet,\Y) \dP(\X_\IndSet^c|\X_\IndSet,\Y)} = \frac{\dQ(\X_\IndSet,\Y) }{\dP(\X_\IndSet,\Y)}
\end{equation*}
Similarly, by the assumption  $\dA(\X_\IndSetJ^c|\X_\IndSetJ,\Y)= \dP(\X_\IndSetJ^c|\X_\IndSetJ,\Y)$, we have 
\begin{equation*}
    w_a(\X,\Y) = \frac{\dA(\X,\Y)}{\dP(\X,\Y)} = \frac{\dA(\X_\IndSetJ,\Y) \dA(\X_\IndSetJ^c|\X_\IndSetJ,\Y)}{\dP(\X_\IndSetJ,\Y) \dP(\X_\IndSetJ^c|\X_\IndSet,\Y)} = \frac{\dA(\X_\IndSetJ,\Y) }{\dP(\X_\IndSetJ,\Y)}
\end{equation*}
Thus, we only need to show 
\begin{equation*}
\frac{\dQ(\X_\IndSet,\Y) }{\dP(\X_\IndSet,\Y)} = \frac{\dA(\X_\IndSetJ,\Y) }{\dP(\X_\IndSetJ,\Y)}
\end{equation*}
Our approach is to show that the two ratios all satisfy a system of linear equations, which, however, should have only a unique solution. 
To see this, let us first note that, for any $\bar{\x}$, we have  
\begin{equation*}
\begin{split}
    \dQ(\X_{\IndSet^c\cap \mathcal{J}^c}, \X_{\IndSet\cup \mathcal{J}}=\x_{\IndSet\cup \mathcal{J}}) =& \sum_{y=1}^{d}     \dQ(\X_{\IndSet^c\cap \mathcal{J}^c}, \X_{\IndSet\cup \mathcal{J}}=\x_{\IndSet\cup \mathcal{J}},\Y=\y) \\
    =& \sum_{y=1}^{d}     \dQ(\X_{\IndSet^c\cap \mathcal{J}^c}, \X_{ \mathcal{J}}=\x_{\mathcal{J}}|\X_{\IndSet}=\x_{\IndSet},\Y=\y) \dQ(\X_{\IndSet}=\x_{\IndSet},\Y=\y)\\
    =& \sum_{y=1}^{d}     \dP(\X_{\IndSet^c\cap \mathcal{J}^c}, \X_{ \mathcal{J}}=\x_{\mathcal{J}}|\X_{\IndSet}=\x_{\IndSet},\Y=\y) \dQ(\X_{\IndSet}=\x_{\IndSet},\Y=\y)\\ 
    =& \sum_{y=1}^{d}     \dP(\X_{\IndSet^c\cap \mathcal{J}^c}, \X_{ \IndSet\cup \mathcal{J}}=\x_{\IndSet\cup \mathcal{J}},\Y=\y) \frac{\dQ(\X_{\IndSet}=\x_{\IndSet},\Y=\y)}{\dP(\X_{\IndSet}=\x_{\IndSet},\Y=\y)}\\     
\end{split}
\end{equation*}
Here, the first equation is by the total probability rule, the second equation is by the conditional probability rule, the third equation is by the $s$-\systemnameISS{} assumption, and the last equation is by conditional probability rule again. Similarly, we can obtain 
\begin{equation*}
\begin{split}
    \dA(\X_{\IndSet^c\cap \mathcal{J}^c}, \X_{\IndSet\cup \mathcal{J}}=\x_{\IndSet\cup \mathcal{J}}) =& \sum_{y=1}^{d}     \dA(\X_{\IndSet^c\cap \mathcal{J}^c}, \X_{\IndSet\cup \mathcal{J}}=\x_{\IndSet\cup \mathcal{J}},\Y=\y) \\
    =& \sum_{y=1}^{d}     \dA(\X_{\IndSet^c\cap \mathcal{J}^c}, \X_{ \IndSet}=\x_{\IndSet}|\X_{\IndSetJ}=\x_{\IndSetJ},\Y=\y) \dQ(\X_{\IndSetJ}=\x_{\IndSetJ},\Y=\y)\\
    =& \sum_{y=1}^{d}     \dP(\X_{\IndSet^c\cap \mathcal{J}^c}, \X_{ \IndSet}=\x_{\IndSet}|\X_{\IndSetJ}=\x_{\IndSetJ},\Y=\y) \dQ(\X_{\IndSetJ}=\x_{\IndSetJ},\Y=\y)\\
    =& \sum_{y=1}^{d}     \dP(\X_{\IndSet^c\cap \mathcal{J}^c}, \X_{ \IndSet\cup \mathcal{J}}=\x_{\IndSet\cup \mathcal{J}},\Y=\y) \frac{\dQ(\X_{\IndSetJ}=\x_{\IndSetJ},\Y=\y)}{\dP(\X_{\IndSetJ}=\x_{\IndSetJ},\Y=\y)}\\     
\end{split}
\end{equation*}
where the first equation is by the total probability rule, the second equation is by the conditional probability rule, the third equation is by the assumption that $\dA(\X_{\IndSetJ^c}|\X_\IndSetJ,\Y)=\dP(\X_\IndSetJ|\X_\IndSetJ,\Y)$, and the last equation is by conditional probability rule again.
Note that we have assumed that $\dA(\X)=\dQ(\X)$. Thus, we must have 
\begin{equation*}
\begin{split}
&\sum_{y=1}^{d}     \dP(\X_{\IndSet^c\cap \mathcal{J}^c}, \X_{ \IndSet\cup \mathcal{J}}=\x_{\IndSet\cup \mathcal{J}},\Y=\y) \frac{\dQ(\X_{\IndSet}=\x_{\IndSet},\Y=\y)}{\dP(\X_{\IndSet}=\x_{\IndSet},\Y=\y)}    \\
=& \sum_{y=1}^{d}     \dP(\X_{\IndSet^c\cap \mathcal{J}^c}, \X_{ \IndSet\cup \mathcal{J}}=\x_{\IndSet\cup \mathcal{J}},\Y=\y) \frac{\dA(\X_{\IndSetJ}=\x_{\IndSetJ},\Y=\y)}{\dP(\X_{\IndSetJ}=\x_{\IndSetJ},\Y=\y)}\\     
\end{split}
\end{equation*}
which is simply

\begin{equation*}
\begin{split}
&\sum_{y=1}^{d}     \dP(\X_{\IndSet^c\cap \mathcal{J}^c}, \X_{ \IndSet\cup \mathcal{J}}=\x_{\IndSet\cup \mathcal{J}},\Y=\y) \left(\frac{\dQ(\X_{\IndSet}=\x_{\IndSet},\Y=\y)}{\dP(\X_{\IndSet}=\x_{\IndSet},\Y=\y)} - \frac{\dA(\X_{\IndSetJ}=\x_{\IndSetJ},\Y=\y)}{\dP(\X_{\IndSetJ}=\x_{\IndSetJ},\Y=\y)}\right)=0    \\
\end{split}
\end{equation*}
By the assumption that $\{\dP(\X_{\mathcal{J}^c\cap\mathcal{K}^c},\X_{\mathcal{J}\cup\IndSet}=\x_{\mathcal{J}\cup \IndSet},\Y=\y)\}_{y=1}^{d}$ are linearly independent, the above system of equations implies 
$\frac{\dQ(\X_{\IndSet}=\x_{\IndSet},\Y=\y)}{\dP(\X_{\IndSet}=\x_{\IndSet},\Y=\y)} - \frac{\dA(\X_{\IndSetJ}=\x_{\IndSetJ},\Y=\y)}{\dP(\X_{\IndSetJ}=\x_{\IndSetJ},\Y=\y)}=0$.
Note that this holds for any $\x$.
Thus, it is simply
\begin{equation*}
\frac{\dQ(\X_\IndSet,\Y) }{\dP(\X_\IndSet,\Y)} = \frac{\dA(\X_\IndSetJ,\Y) }{\dP(\X_\IndSetJ,\Y)}
\end{equation*}
That is to say, $\w^*(\X,\Y)=w_a(\X,\Y)$ and thus $\dA(\X,\Y) = \dQ(\X,\Y)$, which is a contradiction. Thus, the assumption is incorrect and $(\dP,\dQ)$ is identifiable, which finishes the proof.
\end{proof}

\subsection{Proof of Theorem \ref{thm:SparseShift:connection}}
\begin{proof}
Proving the first half statement is straightforward:  suppose $(\dP,\dQ)$ is under label shift. Then by definition, $\dQ(\X|\Y) = \dP(\X|\Y)$.
That is basically $\dQ(\X_{[d]}|\Y,\X_{\varnothing}) = \dP(\X_{[d]}|\Y,\X_\varnothing)$, which corresponds to 0-\systemnameISS{} with $I=\varnothing$.

Next we show the proof for the second half.
Suppose  $(\dP,\dQ)$ is under sparse covariate shift, i.e.,  $\dQ(\X_{\IndSet^c}|\X_{\IndSet})=\dP(\X_{\IndSet^c}|\X_{\IndSet})$ for some $I$ with size $\sparse<d$, and $\dQ(\Y|\X)=\dP(\Y|\X)$.
Adopting the definition of conditional probability, $\dQ(\Y|\X)=\dP(\Y|\X)$ can be rewritten as 
\begin{equation*}
 \frac{\dQ(\Y,\X)}{\dQ(\X)} = \frac{\dP(\Y,\X)}{\dP(\X)}   
\end{equation*}
By definition of conditional probability, $\dP(\X) = \dP(\X_{\IndSet}) \cdot \dP(\X_{\IndSet^c}|\X_{\IndSet})$ and $\dQ(\X) = \dQ(\X_{\IndSet}) \cdot \dQ(\X_{\IndSet^c}|\X_{\IndSet})$, and thus we have
\begin{equation*}
 \frac{\dQ(\Y,\X)}{\dQ(\X_{\IndSet}) \cdot \dQ(\X_{\IndSet^c}|\X_{\IndSet})} = \frac{\dP(\Y,\X)}{\dP(\X_{\IndSet}) \cdot \dP(\X_{\IndSet^c}|\X_{\IndSet})}   
\end{equation*}
By the assumption $\dQ(\X_{\IndSet^c}|\X_{\IndSet})=\dP(\X_{\IndSet^c}|\X_{\IndSet})$, we can simplify this as
\begin{equation*}
 \frac{\dQ(\Y,\X)}{\dQ(\X_{\IndSet}) } = \frac{\dP(\Y,\X)}{\dP(\X_{\IndSet}) }   
\end{equation*}
By definition of conditional probability, this is basically 
\begin{equation*}
 \dQ(\Y,\X_{\IndSet^c}|\X_{\IndSet}) = \dP(\Y,\X_{\IndSet^c}|\X_{\IndSet})
\end{equation*}
which means $(\dP,\dQ)$ is under $\sparse$-\systemnameISS{}.

Now we show the last piece of the statement by construction.
Consider the case of $d=2$. $\X_1,\X_2,\Y$ are all binary variables.
$\DP$ is generated as follows: $\Y$ is first generated from Bernoulli distribution  $Bern(0.5)$. If $\Y=0$, $\X_1,\X_2$ are independently generated from $Bern(0.7)$ and $Bern(0.6)$.
If $\Y=1$, $\X_1,\X_2$ are independently generated from $Bern(0.1)$ and $Bern(0.2)$.
$\DQ$ is generated as follows: $\Y$ is first generated from $Bern(0.6)$. If $\Y=0$, $\X_1,\X_2$ are independently generated from $Bern(0.5)$ and $Bern(0.6)$.
If $\Y=1$, $\X_1,\X_2$ are independently generated from $Bern(0.5)$ and $Bern(0.2)$.

It is easy to see that $(\dP,\dQ)$ is under $1$-\systemnameISS{} with associated shift index set $\IndSet=\{1\}$, since $\X_2$ is independent of $\X_{1}$ and only depends on $\Y$.
However, 
\begin{equation*}
    \dQ(\X_1=1|\Y=1) = 0.5\not= 0.1 = \dP(\X_1=1|\Y=1)
\end{equation*}
and thus $(\dP,\dQ)$ is not under label shift. In addition, 
\begin{equation*}
    \dQ(\Y=1|\X_1=1,\X_2=1) = \frac{0.6\times 0.5\times 0.2}{0.6\times 0.5\times 0.2+0.4\times 0.5\times 0.6} = \frac{1}{3}
\end{equation*}
\begin{equation*}        \dP(\Y=1|\X_1=1,\X_2=1) = \frac{0.5\times 0.1\times 0.2}{0.5\times 0.1\times 0.2+0.5\times 0.7\times 0.6} = \frac{1}{22}\not=\dQ(\Y=1|\X_1=1,\X_2=1)
\end{equation*}
and thus $(\dP,\dQ)$ is not under covaraite shift, which completes the proof.
\end{proof}

\subsection{Proof of Theorem \ref{Thm:SparseShift:AlgConvergence}}
\begin{proof}
We prove the statement via three main steps.
First, we show that given full access to the distribution, the optimization over the marginal mass functions are sufficient to obtain the correct shifted features and weights. Next, we demonstrate that a large enough number of samples ensures the identified index set stays the same as when full access to the distribution is given with high probability. Finally, we can prove that with high probability, the learned weight function with large number of samples is close to the learned weight function when full distribution is known. 

To proceed, let us introduce a few more notations for convenience. 
\begin{itemize}
    \item The distance used by \systemnameShiftAtt{}-d: $\dist(J,\w_J) \triangleq \sum_{\kappa:J\subseteq\kappa,|\kappa|=2\sparse  }^{}\sum_{\bar{f}=1  }^{L}\sum_{
    \x_\kappa \in \mathcal{X}_\kappa    }^{}\|\dQ(\x_\kappa,\bar{f}) - \sum_{\bar{y}=1}^{L} \w_J(\x_J,\bar{y}) \cdot \dP( \x_\kappa, \bar{f}, \bar{y}) \|_2^2$.
    \item The empirical distance used by \systemnameShiftAtt{}-d with finite samples: $\disthat(J,\w_J) \triangleq \sum_{\kappa:J\subseteq\kappa,|\kappa|=2\sparse  }^{}\sum_{\bar{f}=1  }^{L}\sum_{
    \x_\kappa \in \mathcal{X}_\kappa    }^{}\|\dQhat(\x_\kappa,\bar{f}) - \sum_{\bar{y}=1}^{L} \w_J(\x_J,\bar{y}) \cdot \dPhat( \x_\kappa, \bar{f}, \bar{y}) \|_2^2$
    
    \item The optimal weight function when the shifted feature set is fixed to $J$: $   \w_\IndSetJ^* \triangleq \arg \min_{\w_J(\X,y)} \sum_{\kappa:J\subseteq\kappa,|\kappa|=2\sparse  }^{}\sum_{\bar{f}=1  }^{L}\sum_{
    \x_\kappa \in \mathcal{X}_\kappa    }^{}\|\dQ(\x_\kappa,\bar{f}) - \sum_{\bar{y}=1}^{L} \w_J(\x_J,\bar{y}) \cdot \dP( \x_\kappa, \bar{f}, \bar{y}) \|_2^2$

    \item The optimal weight function when the shifted feature set is fixed to $J$ and full distribution is available: $   \w_\IndSetJ^* \triangleq \arg \min_{\w_J(\X,y)} \sum_{\kappa:J\subseteq\kappa,|\kappa|=2\sparse  }^{}\sum_{\bar{f}=1  }^{L}\sum_{
    \x_\kappa \in \mathcal{X}_\kappa    }^{}\|\dQ(\x_\kappa,\bar{f}) - \sum_{\bar{y}=1}^{L} \w_J(\x_J,\bar{y}) \cdot \dP( \x_\kappa, \bar{f}, \bar{y}) \|_2^2$
    
    \item The optimal weight function when the shifted feature set is fixed to $J$ and only finite samples are available: $   \what_\IndSetJ^* \triangleq \arg \min_{\w_J(\X,y)} \sum_{\kappa:J\subseteq\kappa,|\kappa|=2\sparse  }^{}\sum_{\bar{f}=1  }^{L}\sum_{
    \x_\kappa \in \mathcal{X}_\kappa    }^{}\|\dQhat(\x_\kappa,\bar{f}) - \sum_{\bar{y}=1}^{L} \w_J(\x_J,\bar{y}) \cdot \dPhat( \x_\kappa, \bar{f}, \bar{y}) \|_2^2$
    
\end{itemize}

Next we offer a few useful lemmas before giving the full proof.

\begin{lemma}\label{lemma:SparseShift:optsmallerror}
Let $\Omega$ be a compact set, and $\z^*_1, \z^*_2$ be the optimal solution to the problems
\begin{equation*}
    \min_{\z \in \Omega} f_1(\z)
\end{equation*}
and
\begin{equation*}
    \min_{\z \in \Omega} f_2(\z)
\end{equation*}
where  $f_1(\cdot), f_2(\cdot)$ are two functions defined on $\Omega$ such that $\left|f_1(\z)-f_2(\z)\right|\leq \Delta, \forall \z\in\Omega$.
Then we have $|f_1(\z^*_1)- f_2(\z^*_2)|\leq \Delta$.
If $f_1$ is strongly convex with parameter $\lambda$ and $\Omega$ is the full real vector space, then $\|\z^*_1-\z^*_2 \|_2^2\leq \frac{4\Delta}{\lambda}$.
\end{lemma}
\begin{proof}
We start with the first half statement. By $\left|f_1(\z)-f_2(\z)\right|\leq \Delta, \forall \z\in\Omega$, we have $f_1(\z^*_1) \leq \Delta + f_2(\z^*_1)$. Subtracting $f_2(\z_2^*)$ from both sides gives 
\begin{equation*}
f_1(\z^*_1) - f_2(\z^*_2) \leq \Delta + f_2(\z^*_1) - f_2(\z^*_2)
\end{equation*}
Observe that $\z_2^*$ is the optimal solution to minimizing $f_2(\cdot)$ on the set $\Omega$. Thus, $f_2(\z_1^*)\leq f_2(\z_2^*)$ must hold. Thus, the above inequality becomes
\begin{equation*}
f_1(\z^*_1) - f_2(\z^*_2) \leq \Delta 
\end{equation*}
By symmetry of the two functions, we can obtain \begin{equation*}
f_2(\z^*_2) - f_1(\z^*_1) \leq \Delta 
\end{equation*} 
Combining the two inequalities gives 
\begin{equation*}
|f_2(\z^*_2) - f_1(\z^*_1)| \leq \Delta 
\end{equation*} 

Next let us turn to the second half. We first notice that $|f_1(\z_1^*)-f_2(\z_2^*)|\leq 2\Delta$. To see this, we can decompose the difference  $f_1(\z_2^*)-f_2(\z_1^*)$ as 
\begin{equation*}
f_1(\z_2^*)-f_1(\z_1^*) = f_1(\z_2^*) - f_2(\z_2^*)+ f_2(\z_2^*)-f_2(\z_1^*)+f_2(\z_1^*) -f_1(\z_1^*)    
\end{equation*}
Here, $f_1(\z_2^*) - f_2(\z_2^*)\leq \Delta$ and $f_2(\z_1^*) -f_1(\z_1^*) \leq \Delta $ by the assumption $\left|f_1(\z)-f_2(\z)\right|\leq \Delta, \forall \z\in\Omega$.  
$\z_2^*$ is the optimal solution to minimizing $f_2(\cdot)$, and thus $f_2(\z_2^*)-f_2(\z_1^*)\leq 0$.
Therefore, combining all those leads to 
\begin{equation*}
f_1(\z_2^*)-f_1(\z_1^*) = f_1(\z_2^*) - f_2(\z_2^*)+ f_2(\z_2^*)-f_2(\z_1^*)+f_2(\z_1^*) -f_1(\z_1^*)    \leq 2\Delta
\end{equation*}
Meanwhile, $f_1(\z_2^*)-f_1(\z_1^*)\geq 0$ as $\z_1^*$ is the optimal solution to minimizing $f_1(\cdot)$. That is to say, \begin{equation*}
|f_1(\z_2^*)-f_1(\z_1^*)| \leq 2\Delta
\end{equation*}
Note that $f_1(\cdot)$ is a strongly convex function with parameter $\lambda$. Thus, we have 
\begin{equation*}
    f_1(\z') \geq f_1(\z) + <\frac{\partial f_1(\z)}{\partial \z}, \z'-\z> +\frac{\lambda}{2} \| \z'-\z \|_2^2
\end{equation*}
for any $\z',\z$.
Now let us set $\z=\z_1^*, \z'=\z_2^*$. Remember that $\z_1^*$ is the optimal solution to minimizing $f_1(\cdot)$ and $\Omega$ is the full space. Thus, $\frac{\partial f_1(\z_1^*)}{\partial \z}=0$.
As a result, the above inequality becomes
\begin{equation*}
    f_1(\z^*_2) - f_1(\z^*_1) \geq  \frac{\lambda}{2} \| \z^*_2-\z_1^* \|_2^2
\end{equation*}
By $|f_1(\z_2^*)-f_1(\z_1^*)| \leq 2\Delta$, we obtain
\begin{equation*}
    2\Delta \geq  \frac{\lambda}{2} \| \z^*_2-\z_1^* \|_2^2
\end{equation*}
Rearranging the terms gives
\begin{equation*}
    \| \z^*_2-\z_1^* \|_2^2 \leq \frac{4\Delta}{\lambda}
\end{equation*}
which completes the proof.
\end{proof}

\begin{lemma}\label{lemma:SparseShift:uniquedistance}
Suppose the source and target are under exact $s$-\systemnameISS{}, and for any set $\mathcal{J} \subset [d], |\mathcal{J}|\leq s$ and any $\x\in \mathcal{X}$, the   marginal probability mass functions $\{\dP(f(\X),\X_{\mathcal{J}\cup\IndSet}=\x_{\mathcal{J}\cup \IndSet},\Y=\y)\}_{y=1}^{d}$ are linearly independent. Then  $\dist(J,\w_J) =0$ if and only if $J = I$ and $\w_J(\X_J,\Y)=\w^*(\X,\Y)$.
\end{lemma}

\begin{proof}
Let us first relate the marginal density functions on the target domain to those on the source domain.
Recall that 
\begin{equation*}
    \w^*(\X,\Y) =\frac{\dQ(\X,\Y)}{\dP(\X,\Y)}
\end{equation*}
denote the true weights between the target and source density functions. 
Since the shift is only due to $\X_\IndSet$, $ \w^*(\X,\Y)$ only depends on $\X_\IndSet$. 
Abusing the notation a little bit, we use $\w^*(\X_\IndSet,\Y)$ to denote $\w^*(\X,\Y)$.
Now we can write  $\dQ(\X,\Y) = \w^*(\X_\IndSet,\Y) \cdot \dP(\X,\Y)$.
Thus, the marginal distribution of $(\X_\kappa, f(\X))$ on the target domain can be written as
\begin{equation*}
\begin{split}
    {\dQ}(\x_\kappa, \f) =& \sum_{\y=1}^{L} \sum_{\z:\z_\kappa=\x_\kappa,f(\z)=\f} \dQ(\X=\z,\Y=\y) \\
    = &     \sum_{\y=1}^{L} \sum_{\z:\z_\kappa=\x_\kappa,f(\z)=\f} \w^*(\z_\IndSet,\y) \cdot \dP(\X=\z,\Y=\y)
\end{split}
\end{equation*}
Now let us consider the two directions of the statement separately.

\begin{itemize}
    \item $J = I$ and $\w_J(\X_J,\Y)=\w^*(\X,\Y)$ $\implies$ $\dist(J,\w_J) =0$: In this case, $\IndSet =\IndSetJ \subseteq \kappa$ and thus $\z_\IndSet$ is forced to be $\x_\kappa$.
    Hence, the marginal distribution becomes
    
\begin{equation*}
    \begin{split}
    {\dQ}(\x_\kappa, \f) 
    = &     \sum_{\y=1}^{L} \sum_{\z:\z_\kappa=\x_\kappa,f(\z)=\f} \w^*(\z_\IndSet,\y) \cdot \dP(\X=\z,\Y=\y)\\
    =& \sum_{\y=1}^{L} \w^*(\x_\IndSet,\y) \sum_{\z:\z_\kappa=\x_\kappa,f(\z)=\f}  \dP(\X=\z,\Y=\y) \\
    =& \sum_{\y=1}^{L} \w^*(\x_\IndSet,\y)  \dP(\X_\kappa=\x_\kappa,f(\X)=\f,\Y=\y) \\    
    =& \sum_{\y=1}^{L} \w^*(\x_\IndSetJ,\y)  \dP(\x_\kappa,\f,\y)
\end{split}
\end{equation*}
where the second equation is because $\z_\IndSet$ is fixed and does not depend on the inner summation, the third is by applying definition of conditional probability, and the last is simply change of notations.  
The above equation is simply 
\begin{equation*}
    \begin{split}
    {\dQ}(\x_\kappa, \f) - \sum_{\y=1}^{L} \w^*(\x_\IndSetJ,\y)  \dP(\x_\kappa,\f,\y)  
    = & 0
\end{split}
\end{equation*}
which holds for every $\kappa$. Since $\w_J(\X_J,\Y)=\w^*(\X_J,\Y)$, it is equivalent to 
\begin{equation*}
    \begin{split}
    {\dQ}(\x_\kappa, \f) - \sum_{\y=1}^{L} \w_J(\x_\IndSetJ,\y)  \dP(\x_\kappa,\f,\y)  
    = & 0
\end{split}
\end{equation*}
Thus, summing over the square of it also leads to 0, i.e., 
\begin{equation*}
\dist(J,\w_J) \triangleq \sum_{\kappa:J\subseteq\kappa,|\kappa|=2s  }^{}\sum_{\bar{f}=1  }^{L}\sum_{
    \x_\kappa \in \mathcal{X}_\kappa    }^{}\|\dQ(\x_\kappa,\bar{f}) - \sum_{\bar{y}=1}^{L} \w_J(\x_J,\bar{y}) \cdot \dP( \x_\kappa, \bar{f}, \bar{y}) \|_2^2=0.    
\end{equation*}

    \item  $\dist(J,\w_J) =0$ $\implies$ $\IndSet=\IndSetJ, \w_J(\X_J,\Y)=\w^*(\X,\Y)$: $\dist(J,\w_J) =0$ implies that \begin{equation*}
    \begin{split}
    {\dQ}(\x_\kappa, \f) - \sum_{\y=1}^{L} \w_J(\x_\IndSetJ,\y)  \dP(\x_\kappa,\f,\y)  
    = & 0
\end{split}
\end{equation*}
holds for each $\kappa$. In particular, consider $\kappa$ that contains both $\IndSet$ and $\IndSetJ$. Then $\z_\kappa=\x_\kappa$ implies $\z_\IndSet=\x_\IndSet$. Hence, the marginal distribution can be written as 
\begin{equation*}
    \begin{split}
    {\dQ}(\x_\kappa, \f) 
    = &     \sum_{\y=1}^{L} \sum_{\z:\z_\kappa=\x_\kappa,f(\z)=\f} \w^*(\z_\IndSet,\y) \cdot \dP(\X=\z,\Y=\y)\\
    =& \sum_{\y=1}^{L} \w^*(\x_\IndSet,\y) \sum_{\z:\z_\kappa=\x_\kappa,f(\z)=\f}  \dP(\X=\z,\Y=\y) \\
    =& \sum_{\y=1}^{L} \w^*(\x_\IndSet,\y)  \dP(\X_\kappa=\x_\kappa,f(\X)=\f,\Y=\y) \\    
    =& \sum_{\y=1}^{L} \w^*(\x_\IndSet,\y)  \dP(\x_\kappa,\f,\y)
\end{split}
\end{equation*}
where the second equation is because $\z_\IndSet$ is fixed and does not depend on the inner summation, the third is by applying definition of conditional probability, and the last is simply change of notations.  
Comparing this with the above equation, we end up with 
\begin{equation*}
    \begin{split}
    \sum_{\y=1}^{L} \w^*(\x_\IndSet,\y)  \dP(\x_\kappa,\f,\y)  
     - \sum_{\y=1}^{L} \w_J(\x_\IndSetJ,\y)  \dP(\x_\kappa,\f,\y)  
    = & 0
\end{split}
\end{equation*}
Or alternatively,
\begin{equation*}
    \begin{split}
\sum_{\y=1}^{L} [\w^*(\x_\IndSet,\y)-\w_J(\x_\IndSetJ,\y)]  \dP(\x_\kappa,\f,\y)  
    = & 0
\end{split}
\end{equation*}
which holds for any $\f,\x_\kappa$. By the linear independence assumption, this holds if and only if all the coefficients are 0, i.e., 
$\w^*(\x_\IndSet,\y)-\w_J(\x_\IndSetJ,\y)=0$ for all $\x$ and $\y$.
That is to say, the two functions are identical: $\w^*(\X_\IndSet,\Y)=\w_J(\X_\IndSetJ,\Y)$. 
As $\w_J(\X_J,y)$ and $W_*(\X_J,y)$ are identical, they can only depend on variables in the set $I \cap J$. 
That is to say,  there exists another importance weights $\w_{I\cap J}(\X_{I\cap J},y)$ which results in the same target distribution produced by $W_I(\x_I,y)$. 
By the assumption, the shift among the two distribution is exactly $s$-\systemnameISS{}.
Hence, we have $|I\cap J|=s$. $|I|=s$ and $I\cap J \subseteq  I$ implies $I \cap J = I$ and thus $J \supseteq  I$. $|J|=s$ further implies $J=I$.
\end{itemize}

Therefore,  $\dist(J,\w_J) =0$ $\Leftrightarrow$ $\IndSet=\IndSetJ, \w_J(\X_J,\Y)=\w^*(\X,\Y)$, which completes the proof.
\end{proof}

\begin{lemma}\label{lemma:SparseShift:boundempiricaldistance}
With probability at least $1-\delta$, for any possible $J\subseteq [d], |J|=s$,  we have
\begin{equation*}
\begin{split}
&|\dist(\IndSetJ,\w_\IndSetJ^*)-\disthat(\IndSetJ,\what_\IndSetJ^*) | \\
\leq & 3\cdot (2s\bar{v} )^s M L^2 \sqrt{2 s \log d + s \log \bar{v} +2 \log L + \log 1/\delta} \left( \sqrt{\frac{1}{2\nP}} +
 LM \sqrt{\frac{1}{2\nQ}} \right).
\end{split}
\end{equation*}
and
 \begin{equation*}
 \|\w^*_\IndSetJ-\what^*_\IndSetJ\|_2^2 \leq O\left(M L^2 \left( \sqrt{\frac{\log 1/\delta}{2\nP}} +
 LM \sqrt{\frac{\log 1/\delta}{2\nQ}} \right)\right)
 \end{equation*}
\end{lemma}
\begin{proof}
Let us start by considering a fixed $J$. 
By definition of $\w^*_\IndSetJ$ and $\what_\IndSetJ^*$, we have 
\begin{equation*}
 \dist(\IndSetJ, \w_\IndSetJ^*) =  \min_{\w_J(\X,y)} \sum_{\kappa:J\subseteq\kappa,|\kappa|=2s  }^{}\sum_{\bar{f}=1  }^{L}\sum_{
    \x_\kappa \in \mathcal{X}_\kappa    }^{}\|\dQ(\x_\kappa,\bar{f}) - \sum_{\bar{y}=1}^{L} \w_J(\x_J,\bar{y}) \cdot \dP( \x_\kappa, \bar{f}, \bar{y}) \|_2^2
\end{equation*}
and 
\begin{equation*}
 \disthat(\IndSetJ, \what_\IndSetJ^*) =  \min_{\w_J(\X,y)} \sum_{\kappa:J\subseteq\kappa,|\kappa|=2s  }^{}\sum_{\bar{f}=1  }^{L}\sum_{
    \x_\kappa \in \mathcal{X}_\kappa    }^{}\|\dQhat(\x_\kappa,\bar{f}) - \sum_{\bar{y}=1}^{L} \w_J(\x_J,\bar{y}) \cdot \dPhat( \x_\kappa, \bar{f}, \bar{y}) \|_2^2
\end{equation*}
    
Let us first show that the above two  objective functions are close  for any fixed $\w_\IndSetJ$.
To see this, we first apply difference of two squares to obtain  
\begin{equation}\label{eq:SparseShift:tempSquare}
\begin{split}
& \left({\dQhat}(\x_\kappa,\f)- \sum_{\y=1}^{L}\w_\IndSetJ(\x_\IndSetJ,\y)   {\dPhat}(\x_\kappa,\f,\y) \right)^2 
-  \left({\dQ}(\x_\kappa,\f)- \sum_{\y=1}^{L}\w_\IndSetJ(\x_\IndSetJ,\y)   {\dP}(\x_\kappa,\f,\y) \right)^2    \\
= & \left({\dQhat}(\x_\kappa,\f)+{\dQ}(\x_\kappa,\f)- \sum_{\y=1}^{L}\w_\IndSetJ(\x_\IndSetJ,y)   \left({\dPhat}(\x_\kappa,\f,\y) + {\dP}(\x_\kappa,\f,\y) \right)\right)\\
 & \cdot \left({\dQhat}(\x_\kappa,\f)-{\dQ}(\x_\kappa,\f)- \sum_{\y=1}^{L}\w_\IndSetJ(\x_\IndSetJ,y)   \left({\dPhat}(\x_\kappa,\f,\y) - {\dP}(\x_\kappa,\f,\y) \right)\right)
\end{split}
\end{equation}
Note that all estimated probability mass must be bounded by 1, and by assumption, $|\w_J|\leq M$. 
Thus, 
\begin{equation}\label{eq:SparseShift:temp1}
\begin{split}
& \left|{\dQhat}(\x_\kappa,\f)+{\dQ}(\x_\kappa,\f)- \sum_{\y=1}^{L}\w_\IndSetJ(\x_\IndSetJ,y)   \left({\dPhat}(\x_\kappa,\f,\y) + {\dP}(\x_\kappa,\f,\y) \right)\right|
\leq  2+2LM\leq 3LM
\end{split}
\end{equation}
and 
\begin{equation*}
\begin{split}
& \left|{\dQhat}(\x_\kappa,\f)-{\dQ}(\x_\kappa,\f)- \sum_{\y=1}^{L}\w_\IndSetJ(\x_\IndSetJ,\y)   \left({\dPhat}(\x_\kappa,\f,\y) - {\dP}(\x_\kappa,\f,\y) \right)\right|\\
\leq & \left|{\dQhat}(\x_\kappa,\f)-{\dQ}(\x_\kappa,\f)\right| + \left| \sum_{\y=1}^{L}\w_\IndSetJ(\x_\IndSetJ,\y)   \left({\dPhat}(\x_\kappa,\f,\y) - {\dP}(\x_\kappa,\f,\y) \right)\right|\\
\leq & \left|{\dQhat}(\x_\kappa,\f)-{\dQ}(\x_\kappa,\f)\right| + \sum_{\y=1}^{L}\w_\IndSetJ(\x_\IndSetJ,\y) \left| \left({\dPhat}(\x_\kappa,\f,\y) - {\dP}(\x_\kappa,\f,\y) \right)\right|\\
\leq & \left|{\dQhat}(\x_\kappa,\f)-{\dQ}(\x_\kappa,\f)\right| + M \sum_{\y=1}^{L} \left| \left({\dPhat}(\x_\kappa,\f,\y) - {\dP}(\x_\kappa,\f,\y) \right)\right|
\end{split}
\end{equation*}
Observe that, ${\dQhat}(\x_\kappa,\f)$ is the standard empirical estimation of ${\dQ}(\x_\kappa,\f)$.
Therefore, applying Hoeffding's inequality, we have with probability $1-\delta$, 
\begin{equation*}
|    {\dQhat}(\x_\kappa,\f)-{\dQ}(\x_\kappa,\f) | \leq \sqrt{\frac{
\log 1/\delta }{2\nP}}
\end{equation*}
Similarly, with probability $1-\delta$, 
\begin{equation*}
|    {\dPhat}(\x_\kappa,\f,\y) -{\dP}(\x_\kappa,\f,\y) | \leq \sqrt{\frac{
\log 1/\delta }{2\nQ}}
\end{equation*}
Thus, with probability $1- (\bar{v}^s L +\bar{v}^s L^2)\delta$, the above holds for any $\f,\y,\x_\kappa$.
Thus we have
\begin{equation*}
\begin{split}
& |{\dQhat}(\x_\kappa,\f)-{\dQ}(\x_\kappa,\f)- \sum_{\y=1}^{L}\w_\IndSetJ(\x_\IndSetJ,\y)   ({\dPhat}(\x_\kappa,\f,\y) - {\dP}(\x_\kappa,\f,\y) )|
\leq  \sqrt{\frac{\log 1/\delta}{2\nP}} +
 LM \sqrt{\frac{\log 1/\delta}{2\nQ}}
\end{split}
\end{equation*}
Combing this with inequalities \ref{eq:SparseShift:tempSquare} and \ref{eq:SparseShift:temp1}, we have

\begin{equation*}
\begin{split}
&\left({\dQhat}(\x_\kappa,\f)- \sum_{\y=1}^{L}\w_\IndSetJ(\x_\IndSetJ,\y)   {\dPhat}(\x_\kappa,\f,\y) \right)^2 
-  \left({\dQ}(\x_\kappa,\f)- \sum_{\y=1}^{L}\w_\IndSetJ(\x_\IndSetJ,\y)   {\dP}(\x_\kappa,\f,\y) \right)^2\\
\leq  & 3 LM \left( \sqrt{\frac{\log 1/\delta}{2\nP}} +
 LM \sqrt{\frac{\log 1/\delta}{2\nQ}} \right)
\end{split}
\end{equation*}
Summing over $\kappa, \x_\kappa, \f$, we have 
\begin{equation*}
\begin{split}
& \sum_{\kappa:\IndSetJ \subseteq \kappa \in [d], |K|=2s}\sum_{\f=1}^{L} \sum_{\x_\kappa \in \mathcal{X}_\kappa}                \left({\dQhat}(\x_\kappa,\f)- \sum_{\y=1}^{L}\w_\IndSetJ(\x_\IndSetJ,\y)   {\dPhat}(\x_\kappa,\f,\y) \right)^2 
\\
- & \left({\dQ}(\x_\kappa,\f)- \sum_{\y=1}^{L}\w_\IndSetJ(\x_\IndSetJ,\y)   {\dP}(\x_\kappa,\f,\y) \right)^2  \\
\leq  & (2s)^s\bar{v}^s L \cdot 3 LM \left( \sqrt{\frac{\log 1/\delta}{2\nP}} +
 LM \sqrt{\frac{\log 1/\delta}{2\nQ}} \right)
\end{split}
\end{equation*}
That is to say, for all $\w_\IndSetJ$,  $|\dist(J,\w_\IndSetJ)-\disthat(\IndSetJ,\w_\IndSetJ) | \leq 3 (2s\bar{v} )^s M L^2 \left( \sqrt{\frac{\log 1/\delta}{2\nP}} +
 LM \sqrt{\frac{\log 1/\delta}{2\nQ}} \right)$ with probability  $1- (\bar{v}^s L +\bar{v}^s L^2)\delta$.
 In addition, note that $\dist(\IndSetJ,\w_\IndSetJ)$ is a quadratic function over $\w_\IndSetJ$. By the linear independence assumption, $\dist(\IndSetJ,\w_\IndSetJ)$ must be strongly convex. 
 Now applying Lemma \ref{lemma:SparseShift:optsmallerror}, we have 
 \begin{equation*}
 |\dist(\IndSetJ,\w^*_\IndSetJ)-\disthat(\IndSetJ,\what^*_\IndSetJ) | \leq 3 (2s\bar{v} )^s M L^2 \left( \sqrt{\frac{\log 1/\delta}{2\nP}} +
 LM \sqrt{\frac{\log 1/\delta}{2\nQ}} \right)
 \end{equation*}
 and
  \begin{equation*}
 \|\w_\IndSetJ^*-\what^*_\IndSetJ\|_2^2 \leq \frac{12}{\lambda} (2s\bar{v} )^s M L^2 \left( \sqrt{\frac{\log 1/\delta}{2\nP}} +
 LM \sqrt{\frac{\log 1/\delta}{2\nQ}} \right)
 \end{equation*}
 where $\lambda$ is the parameter corresponding to  the strongly convexity of $\dist(\IndSetJ,\cdot)$.
 This holds for a fixed $J$ with probability  $1- (\bar{v}^s L +\bar{v}^s L^2)\delta$.
 There are $d\choose s$ many possible choices of $J$.
 Thus, with probability at least $1- d^s (\bar{v}^s L +\bar{v}^s L^2)\delta \geq 1- 2 d^s \bar{v}^s L^2 \delta $, the above holds. Replacing $2 d^s \bar{v}^s L^2 \delta$ by $\delta$ gives the desired form.
\end{proof}

Finally we are ready to prove the statement. 
By Lemma \ref{lemma:SparseShift:uniquedistance}, $\dist(\IndSet,\w^*)=0$ and for any $\IndSetJ\not=\IndSet$, we have $\dist(\IndSetJ,\w_\IndSetJ) > \dist(\IndSet,\w^*)$.
Let $c_1 = \min_{J\not=I} d(J,\w_J) >0$, $c_2 = \frac{c_1}{6(2\sparse\bar{v} )^\sparse M L^2}$, and the constant $c=c_2/c_1$. 
Now we make progresses in two steps.

\begin{itemize}
    \item First let us show with high probability, the estimated shifted features match the true shifted features.
    This is equivalent to show, with high probability,   $\disthat(\IndSet,\what^*_\IndSet) < \disthat(\IndSetJ,\what^*_\IndSetJ)$ for any $J\not=I$. To do so, let us note that, for any $J\not= I$,  
\begin{equation*}
\begin{split}
&   \disthat(\IndSet,\what^*_\IndSet) - \disthat(J,\what_\IndSetJ^*) \\ =& \disthat(\IndSet,\what^*_\IndSet)- \dist(\IndSet,\w^*_\IndSet)  - (\disthat(\IndSetJ,\what^*_\IndSetJ)-{\dist}(\IndSetJ,\w^*_\IndSetJ))  + {\dist}(\IndSet,{\w}^*_\IndSet) - {\dist}(\IndSetJ,{\w}_\IndSetJ^*) \\
    = &\disthat(\IndSet,\what^*_\IndSet)- {\dist}(\IndSet,{\w}^*_\IndSet)  - (\disthat(\IndSetJ,\what*_\IndSetJ)-{\dist}(\IndSetJ,{\w}^*_\IndSetJ)) + c_1
\end{split}
 \end{equation*}
By Lemma \ref{lemma:SparseShift:boundempiricaldistance}, with probability $1-\delta$, for any $J$,  \begin{equation*}
\begin{split}
& |\dist(J,\w_\IndSetJ^*)-\disthat(J,\what_\IndSetJ^*) | 
\\
\leq & 3 (2\sparse\bar{v} )^\sparse M L^2 \sqrt{2 \sparse \log d + \sparse \log \bar{v} +2 \log L + \log 1/\delta} \left( \sqrt{\frac{1}{2\nP}} +
 LM \sqrt{\frac{1}{2\nQ}} \right)\\ = &\frac{1}{3} c_2  \sqrt{2 \sparse \log d + \sparse \log \bar{v} +2 \log L + \log 1/\delta} \left( \sqrt{\frac{1}{2\nP}} +
 LM \sqrt{\frac{1}{2\nQ}} \right) 
\end{split}
\end{equation*}
Therefore, we have 
\begin{equation*}
\begin{split}
     & \disthat(\IndSet,\what_\IndSet^*) - \disthat(\IndSetJ,\what^*_\IndSetJ) \\
     \leq & 
 -\frac{2}{3} c_2  \sqrt{2 \sparse \log d + \sparse \log \bar{v} +2 \log L + \log 1/\delta} \left( \sqrt{\frac{1}{2\nP}} +
 LM \sqrt{\frac{1}{2\nQ}}\right) + c_1
\end{split}
 \end{equation*}
By the assumption, 
$\sqrt{2 \sparse \log d + \sparse \log \bar{v} +2 \log L + \log 1/\delta} \left( \sqrt{\frac{1}{2\nP}} +
 LM \sqrt{\frac{1}{2\nQ}}\right) < c_1/c_2$. Hence, we have
\begin{equation*}
\begin{split}
     \disthat(\IndSet,\what_\IndSet^*) - \disthat(\IndSetJ,\what^*_\IndSetJ) \leq 
 &-\frac{2}{3} c_2  \frac{c_1}{c-2} + c_1 = \frac{1}{3} c_1 <0 
\end{split}
 \end{equation*}
 for any $\IndSetJ\not=\IndSet$. Thus, the correct shifted features are selected with probability at least $1-\delta$.

    \item 
Finally we show the learned $\what^*_{\hat{J}}$ is close to the true importance weights $w^*$.
By Lemma \ref{lemma:SparseShift:boundempiricaldistance},
 \begin{equation*}
 \|\w_\IndSetJ-\what_\IndSetJ\|_2^2 \leq O\left(M L^2 \left( \sqrt{\frac{\log 1/\delta}{2\nP}} +
 LM \sqrt{\frac{\log 1/\delta}{2\nQ}} \right)\right)
 \end{equation*}
 with high probability for all $\IndSetJ$. Thus it holds for the selected features $\hat{\IndSetJ}$.
 We have just shown that with high probability, the correct shifted features are selected, i.e., $\hat{\IndSetJ}=\IndSet$.
 Thus, it simply means 
 \begin{equation*}
 \|\w^*-\what^*_{\hat{\IndSetJ}}\|_2^2 \leq O\left(M L^2 \left( \sqrt{\frac{\log 1/\delta}{2\nP}} +
 LM \sqrt{\frac{\log 1/\delta}{2\nQ}} \right)\right)
 \end{equation*}
which completes the proof.
\end{itemize}
\end{proof}

\section{Additional Discussions}
Here we provide additional discussions.
\paragraph{Motivating examples when \systemnameISS{} occurs.}
In Section \ref{Sec:SparseShift:Intro} we give one example when \systemnameISS{} occurs. Now we give two more examples to show how \systemnameISS{} broadly exists in different scenarios. 

\begin{itemize}
    \item Cancer diagnosis: Suppose we wish to build an ML model to diagnose cancer based on patient health records. The model is developed based on labeled dataset in some developed countries. However, when deploying it to hospitals in a developing country, there might be much more young patients, and the cancer rate for the elderly can also increase. Suppose the other features’ distribution remains unchanged given age and cancer diagnosis. Then the distribution shift is naturally an \systemnameISS{}.
    
    \item Toxic text recognition: Consider a mobile app that detects and filters toxic texts based on the content and senders’ information. Due to unexpected events (for example, disappointing football games), the toxic texts rate, as well as the total number of texts, may both significantly increase in some locations at different time periods. The shift of text locations and toxic text rate is thus another example of \systemnameISS{}.
    
\end{itemize}

\paragraph{Understanding  sparse covariate shifts.}
Sparse covariate shift is a special case of covariate shift~\cite{covariate_sugiyama2007}.
It occurs when the shifts are caused by a few variables. 
For example, consider two census datasets collected in two periods. If a large population moved from one city to another between the two periods and everything else remains the same, then there is a sparse covariate shift (location alone). It is also related to Adversarial patches: if adversarial noises are added to a few features (or a small number of pixels in image domains), it also corresponds to the sparse covariate shift.

\section{Experimental Details}\label{sec:SparseShift:experimentdetails}
Here we provide additional experimental details.

\paragraph{Datasets and ML tasks.}
\begin{table}[htbp]
  \centering
  \small
  \caption{Dataset statistics.}
    \begin{tabular}{|c||c|c|c|}
    \hline
    Dataset & \# of instances & \# of features & Shift types \bigstrut\\
    \hline
    \hline
    BANKCHURN & 10000 & 10    & \multirow{3}[6]{*}{Synthetic} \bigstrut\\
\cline{1-3}    COVID-19 & 660787 & 8     &  \bigstrut\\
\cline{1-3}    CREDIT & 29946 & 23    &  \bigstrut\\
    \hline
    EMPLOY & 227871 & 16    & Geography (CA,PR,IA,WI) \bigstrut\\
    \hline
    INCOME & 245783 & 10    & Geography (CO,CA,KS,OH) \bigstrut\\
    \hline
    INSURANCE & 32140 & 19    & Temporality (2014,2016,2018) \bigstrut\\
    \hline
    \end{tabular}%
  \label{tab:SparseShift:Dataset}%
\end{table}%
We use six datasets for evaluation, namely, BANKCHURN~\cite{dataset_bankchurn}, COVID-19~\cite{dataset_covid19}, and CREDIT~\cite{dataset_credit_yeh2009comparisons} for various \systemnameISS{} simulations, and EMPLOY, INCOME, and INSURANCE~\cite{dataset_census_ding2021retiring} for performance evaluation under real world distribution shifts.
BANKCHURN~\cite{dataset_bankchurn} contains 10 features such as gender, age, credit score and balance, and the goal is to predict whether a bank customer may churn. 
COVID-19~\cite{dataset_covid19} is a subset of the  publicly accessible COVID-19 dataset from the Israel government website,  containing both demographic and symptom features.
Here, we select the subset that contains all tested cases in January, 2022, and aim at predicting whether a person tests positive or negative for COVID-19.
CREDIT~\cite{dataset_credit_yeh2009comparisons} includes age, gender, education, bill payments and several other features for 29946 individuals. 
The goal is to predict the default payment. 
EMPLOY, INCOME, and INSURANCE are subset of the public use microdata samples from the US  census~\cite{dataset_census_ding2021retiring}. 
EMPLOY contain 16 features for individual samples from four different states, CA, PR, IA, and WI and our goal is to predict if a person is employed or not.
In INCOME, 245,783 anonymous census record samples from four states including CA, CO, KS, and OH are collected.  
The goal is to predict if a person's income is lower or higher than \$50,000.
INSURANCE contains
32140 individual samples from the state IA collected in year 2014, 2016, and 2018.
The task is to predict whether an individual is covered by an insurance plan. 
The dataset statistics can be found in Table \ref{tab:SparseShift:Dataset}.

\paragraph{Experiment setups.}
All experiments were run on a machine with 20 Intel Xeon E5-2660 2.6 GHz cores, 160 GB RAM, and 80 GB disk with Ubuntu 18.04 LTS as the OS. 
Our prototype was  implemented and tested in Python 3.8.
To apply \systemnameShiftAtt{}-d, we discretized all continuous features. 
To apply \systemnameShiftAtt{}-c, we set the trade-off parameter $\eta=0.001$.
For continuous features, linear functions were used as the basis.
For discrete features, indicate functions were adopted. 
For example, if $\x_1\in\R$ and $\x_2\in\{0,1\}$, then the basis functions consist of three components, $\phi_1(\x,y)=\x_1,\phi_2(\x,y)=\mathbbm{1}_{\x_2=0}$, and $\phi_3(\x,y) = \mathbbm{1}_{\x_2=1}$.
For KLIEP, the maximum number of iterations was set as 2,500. 
 
\paragraph{Effects of shift sparsity.} 
\begin{figure} \centering
\begin{subfigure}[BANKCHURN]{\label{fig:sarsity_a}\includegraphics[width=0.30\linewidth]{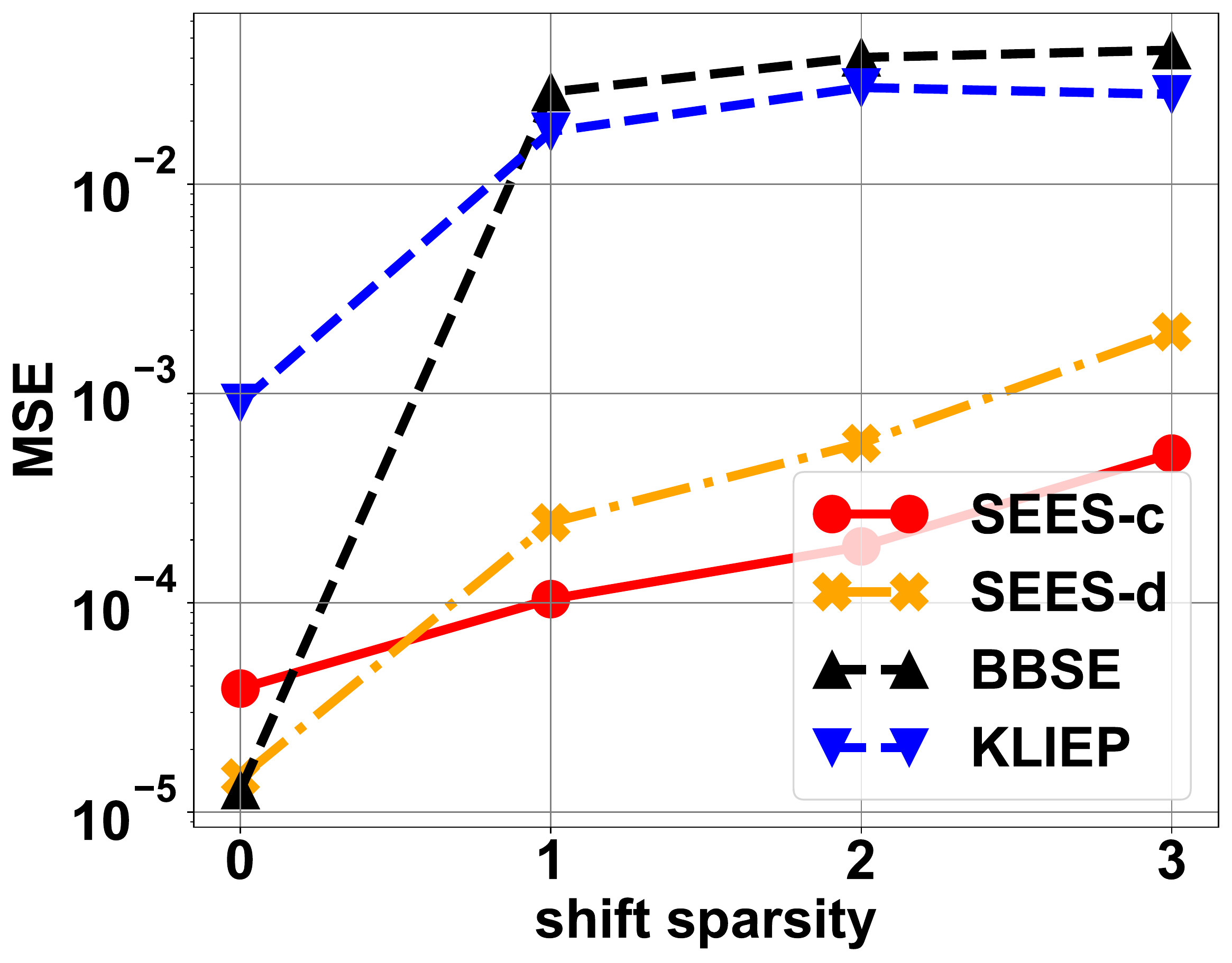}}
\end{subfigure}
\begin{subfigure}[COVID-19]{\label{fig:sparsity_b}\includegraphics[width=0.30\linewidth]{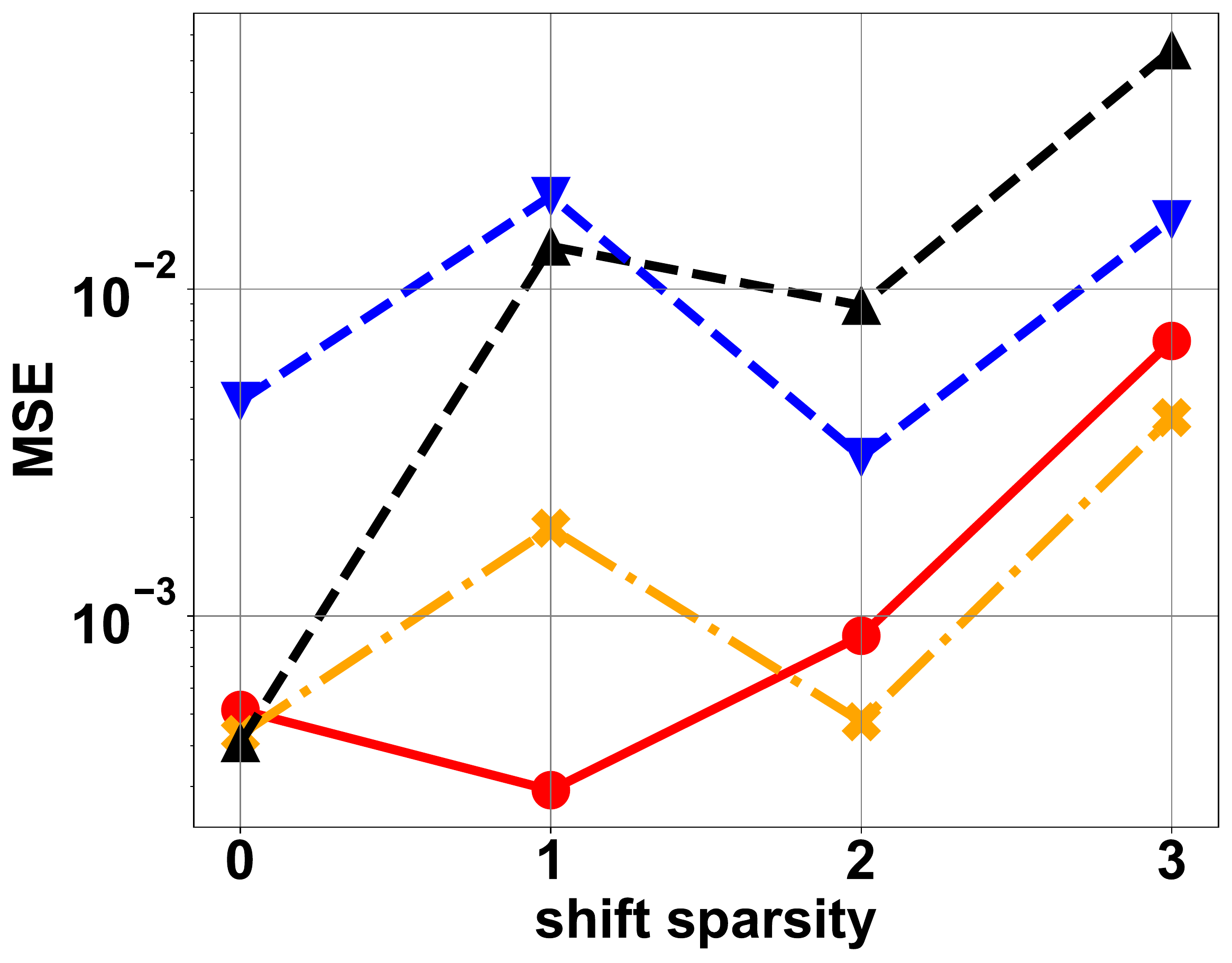}}
\end{subfigure}
\begin{subfigure}[CREDIT]{\label{fig:sparsity_c}\includegraphics[width=0.30\linewidth]{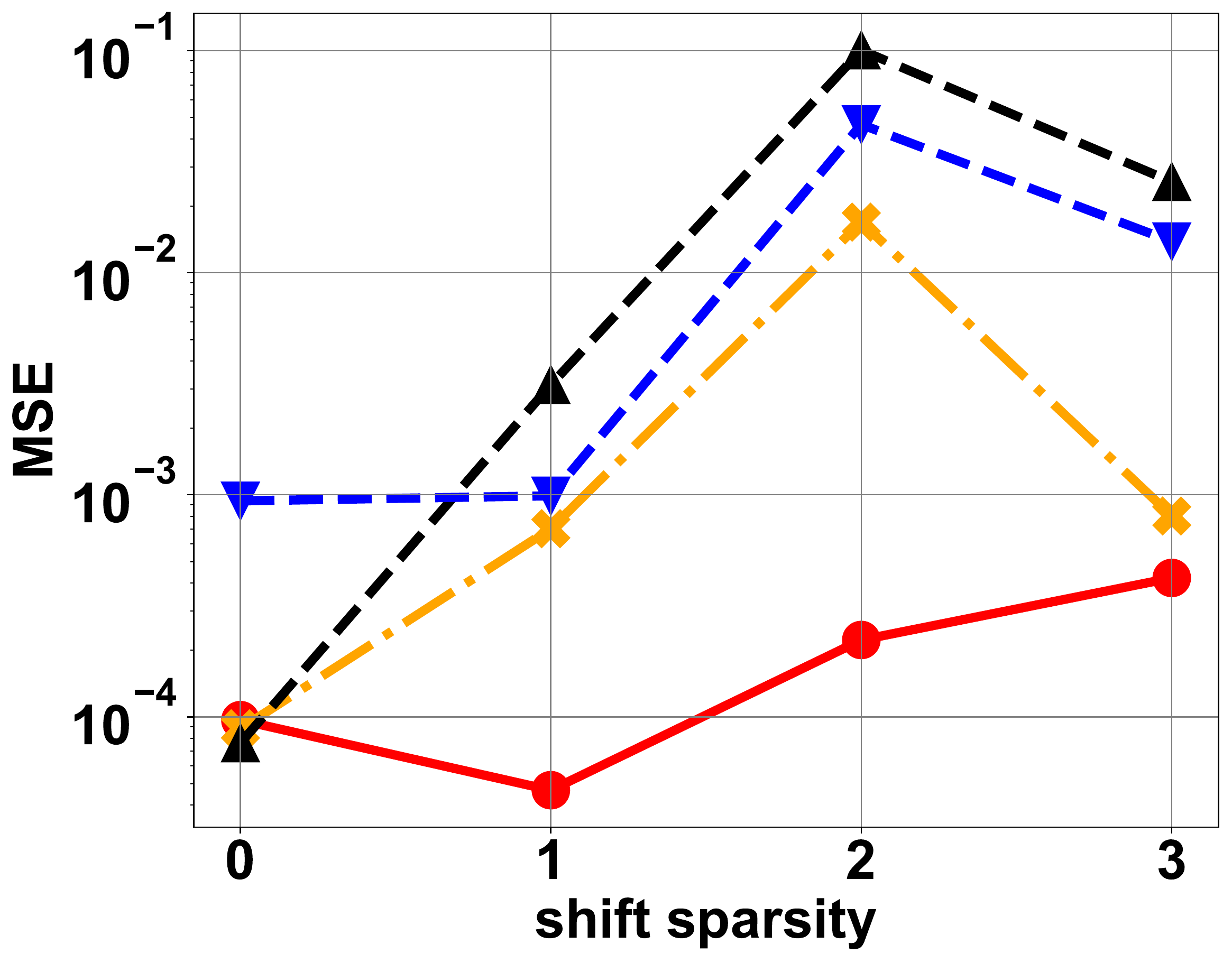}}
\end{subfigure}

\vspace{-0mm}
	\caption{{Effects of shift sparsity. For each dataset, we measure how the $\ell_2$ loss of the estimated accuracy varies as number of shifted features increases given the same sample sizes. Overall, the estimation error of \systemnameShiftAtt{} slowly grows as shift sparsity increases, but is  consistently lower than BBSE and KLIEP. }}\label{fig:SparseShift:sparsityeffect}
\end{figure}

Figure \ref{fig:SparseShift:sparsityeffect} shows how the accuracy gap estimation performance vary as the number of shifted features increase.
For each dataset, we start with shifting no features (label shift), to shifting 1, 2, and 3 features together with labels. 
Specifically, for BANKCHURN, shifts occur for (i) labels alone  (0-\systemnameISS{}), (ii) then both labels and geography feature (1-\systemnameISS{}), (iii) then labels, geography, and gender
(2-\systemnameISS{}), and (iv) finally labels, geography, gender, and credit card owned before.
For COVID-19, we start with shifting label alone (0-\systemnameISS{}), and then incrementally shift features age, gender, and contact risk to simulate 
1-\systemnameISS{}, 2-\systemnameISS{}, and 3--\systemnameISS{}, respectively.
For credits,
0-\systemnameISS{}, 1-\systemnameISS{}, 2-\systemnameISS{}, and 3-\systemnameISS{} correspond to shift in (i) labels, (ii) labels and marriage status, (iii) labels, marriage status, and gender, and (iv) labels, marriage status, gender, and credit balance. Overall, we observe that the estimation error of \systemnameShiftAtt{}-c and \systemnameShiftAtt{}-d slightly increases as the number of shifted features grows, but is consistently lower than that of BBSE and KLIEP.

\paragraph{Robustness to data  randomness.}
In Section \ref{Sec:SparseShift:Experiment} we mainly focus on the average performance metric ($\ell_2$ loss). 
Here we provide 
additional robustness measurement: on the COVID-19 dataset, we repeat the case study experiments 200 times with different random seeds to generate the source and target datasets, and report the variance of the estimated accuracy gap.
As shown in Table \ref{tab:SparseShift:robustness}, we observe that the variance for all of the methods is small: variance of SEES-c and KLIEP are are less than 0.00003, and the variance for SEES-d is 0.003 and for BBSE is 0.0003.

\begin{table}[htbp]
  \centering
  \caption{Variance of the estimated accuracy gap. The values were calculated over 200 experimental runs. Overall, the variance of all methods is small.}
    \begin{tabular}{|c||c|c|c|c|}
    \hline
    Method & SEES-c & SEES-d & KLIEP & BBSE \bigstrut\\
    \hline
    \hline
    Estimated Accuracy Gap Variance & 0.000018 & 0.0033 & 0.00005 & 0.0003 \bigstrut\\
    \hline
    \end{tabular}%
  \label{tab:SparseShift:robustness}%
\end{table}%

\paragraph{Sensitivity of the sparsity parameter in \systemnameShiftAtt{}-d.} 
\systemnameShiftAtt{}-d needs knowledge of the sparsity parameter $s$, and thus a natural question is how sensitive its performance is when the sparsity parameter does not exactly match the true sparsity. 
To study this, we first generate a source-target pair (each containing 10,000 data points) on the COVID-19 dataset where labels and 3 features (age, contact risk, and gender) shift, and then measure the $\ell_2$ loss of the estimated performance gap by \systemnameShiftAtt{}-d with sparsity parameter ranging from 0 to 7 (the number of features).
Here, sparsity parameter being 0, 1, 2, 4, 5, 6, and 7 corresponds to the mismatched case. 
Figure \ref{fig:SparseShift:sparsitysensitivity} summarizes the averaged $\ell_2$ loss over 100 experimental runs.
The randomness comes from the choices of shifted features and the samples from source and target datasets, and the shaded area indicates the standard deviation.
Overall, we observe that \systemnameShiftAtt{}-d is robust to small parameter mismatch: there is little change of the estimation error when the sparsity parameter (2, 3, 4, 5) is close to the true number of shifted features (3).
When the parameter mismatch is too large, a relatively larger change in the estimation error can be observed (though \systemnameShiftAtt{}-d still works better than BBSE even when the model mismatch is large). This is because a too small sparsity parameter restricts the search space, while a too large parameter often incurs an identifiability issue as our theory shows (i.e., different feature-label joint distributions correspond to the same observed target feature distribution). In practice, if the user has a prior belief that the distribution shift is not sparse (i.e. number of shifted features is > $d/2$), then SEES-d may not be appropriate. 

\begin{figure}
    \centering
    \includegraphics[width=0.89\linewidth]{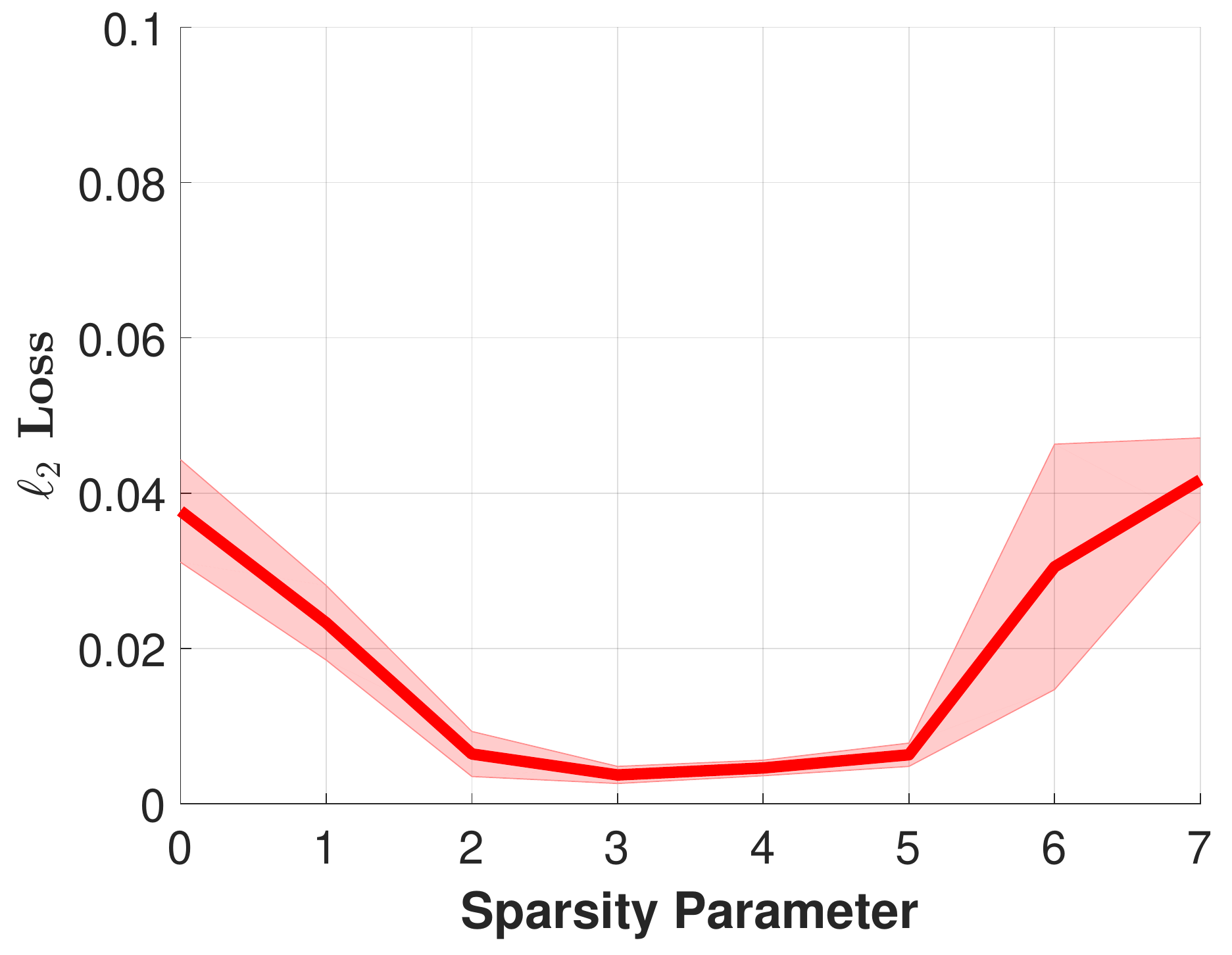}
    \caption{Sensitivity of the sparsity parameter in \systemnameShiftAtt{}-d on the COVID-19 dataset. On a source-target pair where labels and 3 features shift, we measure the $\ell_2$ loss of the estimated model performance gap produced by \systemnameShiftAtt{}-d with sparsity parameter ranging from 0 to 7 (number of features). The performance was averaged over 100 experimental runs, where the randomness comes from choices of shifted features and samples in the source-target pairs and the shaded area indicates the standard deviation.. 
	Overall, \systemnameShiftAtt{}-d is robust to small parameter mismatch: there is little change of the estimation error when the sparsity parameter (2,3,4,5) is close to the true number of shifted features (3).
	In addition, the estimation performance of \systemnameShiftAtt{}-d with small parameter mismatch is also consistently better than all baselines.}
    \label{fig:SparseShift:sparsitysensitivity}
\end{figure}

\eat{
\begin{table}[htbp]
  \centering
  \caption{Sensitivity of the sparsity parameter in \systemnameShiftAtt{}-d on the COVID-19 dataset. On a source-target pair where labels and 3 features shift, we measure the $\ell_2$ loss of the estimated model performance gap produced by \systemnameShiftAtt{}-d with sparsity parameter ranging from 0 to 7 (number of features). Overall, \systemnameShiftAtt{}-d is robust to small parameter mismatch: there is little change of the estimation error when the sparsity parameter (2, 3, 4) is close to the true number of shifted features (3).}
    \begin{tabular}{|c||c|c|c|c|c|c|c|c|}
    \hline
    \small
    sparsity param & 0     & 1     & 2     & 3     & 4     & 5     & 6     & 7 \bigstrut\\
    \hline
    \hline
    $\ell_2$ loss & 0.0307  & 0.0152  & 0.0030  & 0.0026  & 0.0060 & 0.0116 & 0.0371 & 0.0191 \bigstrut\\
    \hline
    \end{tabular}%
  \label{tab:SparseShift:sensitivity}%
\end{table}%
}

\paragraph{Comparison with additional baselines.}
For more in-depth understanding, we compare the performance of  \systemnameShiftAtt{} with an additional baseline DLU ~\cite{park2020calibrated}.
DLU basically adopts discriminative learning on the union of the source and target dataset, and then uses the classifier's prediction to reweigh the source data.  
We adopt it on the COVID-19 dataset for a case study and measure the performance of the weight and accuracy gap estimation for it along with \systemnameShiftAtt{}-c, \systemnameShiftAtt{}-d, BBSE, and KLIEP.  
As shown in Table \ref{tab:SparseShift:AdditionalBaseline}, DLU performs better than KLIEP, but is still much worse than \systemnameShiftAtt{}-c and \systemnameShiftAtt{}-d.
For example, the MSE of its estimated weights is 0.320, while that of \systemnameShiftAtt{}-d is only 0.002. 
Hence, \systemnameShiftAtt{}-c and \systemnameShiftAtt{}-d still lead to the best estimated accuracy gap. 

\begin{table}[htbp]
  \centering
  \small
  \caption{Performance of the weight and accuracy gap estimation on the COVID-19 case study. 
  Overall, \systemnameShiftAtt{}-c and \systemnameShiftAtt{}-d achieve the smallest mean square error (MSE) and largest Pearson correlation coefficient (PCC).
  Thus, their estimated accuracy gap is closest to the true gap.}
    \begin{tabular}{|c||c|c|c|c|}
    \hline
    Method & MSE   & PCC   & Est Gap & True Gap \bigstrut\\
    \hline
    \hline
    SEES-c & 0.003 & 0.996 & 18.4  & \multirow{5}[10]{*}{15.5} \bigstrut\\
\cline{1-4}    SEES-d & \textit{\textbf{0.002}} & \textit{\textbf{0.996}} & \textit{\textbf{16.7}} &  \bigstrut\\
\cline{1-4}    BBSE  & 0.144 & 0.857 & 26.3  &  \bigstrut\\
\cline{1-4}    KLIEP & 0.573 & -0.112 & 1.58  &  \bigstrut\\
\cline{1-4}    DLU   & \textit{0.320} & \textit{0.328} & \textit{0.38} &  \bigstrut\\
    \hline
    \end{tabular}%
  \label{tab:SparseShift:AdditionalBaseline}%
\end{table}%

\paragraph{Robustness across different shifts.}
Real world distribution shifts may vary, and thus it is important to understand how different methods behave when encountering different shifts. To understand this, we study the performance of \systemnameShiftAtt{}-c and \systemnameShiftAtt{}-d along with all baselines (BBSE, KLIEP, and DLU) on the COVID-19 dataset when different shifts occur. Specifically, we generate source-target pair where (i) label shift, (ii) (sparse) covariate shift (feature age), and (iii) 1-\systemnameISS{} (both label and feature age) occurs, separately. 
Table \ref{tab:SparseShift:CompareShifts} gives the squared $\ell_2$ loss of the accuracy gap estimation for all methods.  
Interestingly, we observe that \systemnameShiftAtt{}-c and \systemnameShiftAtt{}-d produce reliable accuracy gap estimation across different shifts, while all baselines are sensitive to shift types. 
For example, KLIEP and DLU achieves decent performance when there is only covariate shift, but their estimation is much worse when label shift occurs.
Similarly, the estimation error of BBSE is small when labels indeed shift but is much worse than other methods when the shift is due to covariate.  
On the other hand, the performance of \systemnameShiftAtt{}-c and \systemnameShiftAtt{}-d is as good as that of the best baseline when label or covariate shifts.
For \systemnameISS{}, the proposed methods achieve significantly better estimation than all baselines. 
Those suggest  that \systemnameShiftAtt{}-c and \systemnameShiftAtt{}-d are more robust to different shifts and thus safer to be deployed in the wild. 

\begin{table}[htbp]
  \centering
  \small
  \caption{Squared $\ell_2$ loss of the estimated accuracy gap on the COVID-19 dataset. \systemnameShiftAtt{}-c and \systemnameShiftAtt{}-d are the only approaches providing reliable estimation across all shifts.}
    \begin{tabular}{|c||c|c|c|}
    \hline
    Method & Label Shift & Covariate Shift & Joint Shift \bigstrut\\
    \hline
    \hline
    SEE-c & 0.0005  & 0.0010  & {0.0003 } \bigstrut\\
    \hline
    SEE-d & 0.0004  & 0.0015  & 0.0019  \bigstrut\\
    \hline
    BBSE  & 0.0004  & 0.0321  & 0.0135  \bigstrut\\
    \hline
    KLIEP & 0.0045  & 0.0008  & 0.0193  \bigstrut\\
    \hline
    DLU  & 0.0029  & 0.0001  & 0.0251  \bigstrut\\
    \hline
    \end{tabular}%
  \label{tab:SparseShift:CompareShifts}%
\end{table}%

\paragraph{More models on datasets with real world shifts.}
Next we provide the performance estimation for more models on datasets with real world shifts. We study the estimation performance for three models, namely, a gradient boosting, a neutral network, and a decision tree.
The maximum depth of the gradient boosting and deicision tree is 50, and the neutral network consists of two layers with 100 hidden units.
  As shown in Table \ref{tab:SparseShift:realshiftfull}, \systemnameShiftAtt{} often provides significant error reduction over the compared baselines BBSE and KLIEP.
\begin{table}[htbp]
  \centering
  \small
  \caption{Root mean square error of estimated accuracy gap (\%) under real shifts for gradient boosting, a neural network, and a decision tree.
  The numbers are averaged over all source-target pairs in each dataset.
  For each dataset and ML model, \systemnameShiftAtt{} provides significant estimation error reduction over baselines.}
    \begin{tabular}{|c|c||c|c|c|c|}
    \hline
    \multirow{2}[4]{*}{Dataset} & \multirow{2}[4]{*}{ML model} & \multicolumn{4}{c|}{Accuracy estimation's $\ell_2$ error (\%)} \bigstrut\\
\cline{3-6}          &       & \systemnameShiftAtt{}-c  & \systemnameShiftAtt{}-d & BBSE  & KLIEP \bigstrut\\
    \hline
    \hline
    \multirow{3}[6]{*}{EMPLOY} & Gradient boosting & \textbf{2.9 } & 3.0   & 5.2   & 5.2  \bigstrut\\
\cline{2-6}          & Neural network & \textbf{5.2 } & 6.0   & 6.2   & 5.6  \bigstrut\\
\cline{2-6}          & Decision tree & 4.0   & \textbf{3.8 } & 5.7   & 5.7  \bigstrut\\
    \hline
    \multirow{3}[6]{*}{INCOME} & Gradient boosting & \textbf{1.9 } & 2.4   & 3.0   & 3.4  \bigstrut\\
\cline{2-6}          & Neural network & \textbf{4.4 } & 6.4   & 9.8   & 7.6  \bigstrut\\
\cline{2-6}          & Decision tree & \textbf{2.4 } & 2.9   & 3.0   & 3.8  \bigstrut\\
    \hline
    \multirow{3}[6]{*}{INSURANCE} & Gradient boosting & \textbf{1.7 } & 2.2   & 2.1   & 5.0  \bigstrut\\
\cline{2-6}          & Neural network & \textbf{2.2 } & 4.7   & 10.8  & 7.3  \bigstrut\\
\cline{2-6}          & Decision tree & \textbf{2.0 } & 2.5   & 2.4   & 5.9  \bigstrut\\
    \hline
    \end{tabular}%
  \label{tab:SparseShift:realshiftfull}%
\end{table}%

\end{document}